\documentclass[9.5pt,journal,compsoc]{IEEEtran}

\usepackage{bbm, dsfont} 
\usepackage{amssymb}
\usepackage[dvipsnames, table]{xcolor}
\usepackage{stmaryrd} 
\usepackage{algorithm}
\usepackage{algorithmic}
\usepackage{multirow}

\usepackage{afterpage} 

\usepackage{array}
\usepackage{eqparbox}

\usepackage{array}
\newcolumntype{P}[1]{>{\centering\arraybackslash}p{#1}} 
\definecolor{Gray}{gray}{0.85}

\usepackage{pifont}
\newcommand{\cmark}{\textcolor{OliveGreen}{\ding{51}}}%
\newcommand{\xmark}{\textcolor{red}{\ding{55}}}%

\usepackage{amsthm}

\newtheorem{lemma}{Lemma}

\usepackage{mathtools} 

\newcommand{\ours}{RED++~}
\newcommand{\oursvir}{RED++}
\newcommand{\REDref}{\cite{red2021}\,}
\newcommand{\ourtitle}{Data-Free Pruning of Deep Neural Networks via Input Splitting and Output Merging}

\ifCLASSOPTIONcompsoc
  \usepackage[nocompress]{cite}
\else
  \usepackage{cite}
\fi
\ifCLASSINFOpdf
\else
\fi

\begin{document}

%
\title{\ours : \ourtitle}
%
%
%
%

\author{Edouard~Yvinec,
        Arnaud~Dapogny~,
        Matthieu~Cord~
        and~Kevin~Bailly
\IEEEcompsocitemizethanks{\IEEEcompsocthanksitem E. Yvinec is a PhD student on Deep Learning Compression methods at Datakalab, 114 boulevard Malesherbes 75017 Paris and ISIR at Sorbonne Université, 4 Place Jussieu 65, 75005 Paris.\\
\IEEEcompsocthanksitem A. Dapogny is a ML researcher at Datakalab.
\IEEEcompsocthanksitem K. Bailly is the Head of research at Datakalab and associate professor at Sorbonne Université.}}

%
%

\markboth{IEEE TRANSACTIONS ON PATTERN ANALYSIS AND MACHINE INTELLIGENCE}%
{Yvinec \MakeLowercase{\textit{et al.}}: \ourtitle}
%



\IEEEtitleabstractindextext{%
\begin{abstract}
Pruning Deep Neural Networks (DNNs) is a prominent field of study in the goal of inference runtime acceleration. 
In this paper, we introduce a novel data-free pruning protocol \oursvir.
Only requiring a trained neural network, and not specific to DNN architecture, we exploit an adaptive data-free scalar hashing which exhibits redundancies among neuron weight values.
We study the theoretical and empirical guarantees on the preservation of the accuracy from the hashing as well as the expected pruning ratio resulting from the exploitation of said redundancies.
We propose a novel data-free pruning technique of DNN layers which removes the input-wise redundant operations. 
This algorithm is straightforward, parallelizable and offers novel perspective on DNN pruning by shifting the burden of large computation to efficient memory access and allocation. 
We provide theoretical guarantees on \ours performance and empirically demonstrate its superiority over other data-free pruning methods and its competitiveness with data-driven ones on ResNets, MobileNets and EfficientNets.
\end{abstract}

}

\maketitle

\IEEEdisplaynontitleabstractindextext

%
\IEEEpeerreviewmaketitle

\IEEEraisesectionheading{\section{Introduction}\label{sec:introduction}}

%
%
%
%

%
%
%
%
\IEEEPARstart{D}{eep} Neural Networks (DNNs) are the corner stone of most machine learning solutions and applications introduced in the recent years. 
Specifically in computer vision, Convolutional Neural Networks (CNNs) achieve outstanding performance on various tasks such as object classification \cite{he2016deep}, detection \cite{he2017mask} or segmentation \cite{chen2017deeplab}. 
However, the high cost in terms of computational power appears to be limiting deployment, especially on edge devices. 
The pruning problem is defined in \cite{frankle2018lottery} as the search of a sub-architecture within the original network for reducing the computational requirements at inference. 
To this end, many pruning algorithms have been developed as discussed in \cite{cheng2017survey}. 
The usual taxonomy of pruning methods is based on the form of sparsity that is enforced upon the network: if whole neurons (or filters in case of convolutional layers) are removed the method is called \textit{structured}, otherwise it is usually refered to as \textit{unstructured} pruning. 
In practice, structured pruning achieves lower architecture compression but can be more readily leveraged without dedicated inference engine and hardware.

Practically speaking, the pruning community has been focusing its efforts on reducing the number of operations performed without changing their nature (i.e. many multiplications are turned into fewer multiplications).
In this work, we propose to change this paradigm.
To do so, we propose to reduce the number of computations by changing identical operations to identical memory access to a unique operation.

Also, following \cite{richards2014big} the usage of data is a crucial limit for pruning algorithm and many other post-treatment protocols.
In recent years, data usage has become a  growing concern in ethics and laws have been established to protect privacy and human rights.
For instance in health and military services, patients and users data are sensible while trained models may be more easily accessible.
This motivates the design of data-free pruning methods such as Dream \cite{yin2020dreaming} and our previous work RED \REDref.

In this paper, we propose a new method that we call \oursvir, corresponds to the RED++ : data-free pruning of deep neural networks \textit{via} input splitting and output merging. 
This hashing step introduces redundant operations in DNN weight distributions, and we provide theoretical guarantees as well as empirical evidence on the robustness of the hashed predictive function w.r.t. the original one. 
Stemming from these redundancies, we propose to remove layer-wise structural redundancies in DNN: First, we merge together output-wise dependencies. Second, we handle input-wise redundancies with a novel input-wise layer splitting. We provide theoretical guarantees on the expected pruning factor of the proposed method as well as its optimality within the bounds of the assumptions of the pruning framework. The whole method, dubbed \oursvir, achieves outstanding performance on every popular DNN architecture, significantly outperforming other recent data-free methods, and often rivalling data-driven ones. To sum it up, our contributions are:
\begin{itemize}
    \item A data-free adaptive scalar hashing protocol to introduce redundancies in DNN weights from our work \REDref. We provide theoretical guarantees on its efficiency in addition to a data-free method to ensure the accuracy preservation. These results are empirically validated on various tasks and architectures.
    \item A novel data-free redundant operation suppression in DNNs via neuron merging (merge) and per input operation splitting (split). This protocol removes all redundant operations performed by a DNN which makes it optimal. We study the expected pruning from these steps and empirically validate these results.
    \item The \ours method (hashing + merging + splitting) is tested on standard benchmarks and achieves remarkable results, outclassing all previous data-free pruning approaches and rivaling many data-driven methods. We also provide novel state-of-the-art performances on the most recent image transformers.
\end{itemize}

For the sake of clarity we opted for a separate presentation of the hashing and architecture compression steps.
After a presentation of the related work on both hashing and pruning in Section \ref{sec:related}, we extend our work from \REDref on hashing via a thorough theoretical and empirical study of its impact in Section \ref{sec:preliminaries}.
After the discussion of the experiments conducted on DNN weight hashing, we present the mathematical properties of our novel pruning algorithm which leverages operation redundancies in Section \ref{sec:preliminaries_method}.
The summary of the work and future extensions are discussed as a conclusion in Section \ref{sec:conclusion}.
%
%
%
%
\section{Related Work}\label{sec:related}
%
%

\subsection{DNN Hashing}\label{sec:related_hashing}

Before detailing the pruning literature, we put the emphasis on a common aspect of pruning algorithms whether they are magnitude based (e.g. \cite{frankle2018lottery,lin2020dynamic,park2020lookahead,lee2019signal}) or similarity-based (e.g. \cite{srinivas2015data, kim2020neuron}) which is the approximation of the original predictive function. 
To the best of our knowledge, every pruning method modifies the predictive function in the sens that we can create inputs that would lead to different logits between the original and pruned network. 
In \REDref, we proposed to separate the approximation on the predictive function and the modifications on the architecture. 
In deep learning, such approximation is usually formulated as an optimisation problem involving the DNN layer activations over the training database \cite{liu2017deep,wang2019haq,stock2019and}, making these methods data-driven.
An intuitive way to approximate weight values independently of data, is hashing using k-means \cite{lloyd1982least} methods. 
However these methods require heuristics or priors (to determine $k$ for instance) on the weight values distribution.
Another well known way to hash DNN weights is quantization \cite{nagel2019data} which consists in mapping weight values to a finite regular grid.
Authors in \cite{nagel2019data,meller2019same,zhao2019improving} propose to tackle data-free quantization.
However these methods come with drawbacks on the accuracy which are a consequence of the constraint on the regularity of the grid.

In this work we propose a data-free hashing method previously described in RED \REDref which doesn't require any assumptions on the weight distribution. 
We provide an extended empirical study of hashing as a pre-process for pruning as well as theoretical guarantees on the preservation of the accuracy of the predictive function after hashing.
%
%

\subsection{DNN Pruning}
%
%
\hspace{14pt}\textbf{Sparsity: }
As thoroughly studied in \cite{renda2020comparing}, pruning methods are divided into either unstructured or structured approaches. 
On the one hand, unstructured approaches \cite{frankle2018lottery,lin2020dynamic,park2020lookahead,lee2019signal} remove individual weights regardless of their structure (e.g. neurons, filters,...) : hence, they rely on sparse matrices optimization for runtime improvements. 
On the other hand, the so-called structured approaches \cite{liebenwein2019provable, li2016pruning, he2018soft, luo2017thinet} remove specific filters, channels or neurons.
Although the latter usually results in lower pruning ratios, they generally allow significant runtime reduction without further work required.
%
%

\textbf{Heuristics: }
Pruning algorithms usually make assumptions on what characterizes the importance of weights or features. 
The most generic is the \textit{magnitude-based} paradigm in which the method erases weights with the lowest magnitude. 
LDI algorithm \cite{lee2019signal} defines the magnitude as the drift of DNN weights from their initial values during training to select and replace irrelevant weights. 
The Hrank method \cite{lin2020hrank} uses the rank of feature maps as a magnitude measurement. 
The work in \cite{lin2020dynamic} extends the single layer magnitude-based weight pruning to a simultaneous multi-layer optimization, in order to better preserve the representation power during training. 
Other approaches, such as \cite{liebenwein2019provable, meng2020pruning}, apply an absolute magnitude-based pruning scheme which removes a number of channels or neurons but generally causes accuracy drop. 
To address this problem, most of these methods usually fine-tune the pruned model afterwards for enhanced performance \cite{liu2018rethinking,gale2019state,frankle2018lottery}. 
By contrast there exists so called similarity-based pruning methods such as \cite{srinivas2015data, kim2020neuron}.
The idea consists in combining weights based on a measure of their resemblance usually by taking the average or the principal components of similar neurons. 
%
%

\textbf{Data usage: }
Almost all the aforementioned pruning methods can be classified as \textit{data-driven} as they involve, to some extent, the use of a training database. 
Some approaches, such as \cite{liebenwein2019provable, meng2020pruning}, train an over-parameterized model from which they remove a number of channels or neurons but generally causes accuracy drop. 
To address this problem, most of these methods usually fine-tune this pruned model for enhanced performance \cite{liu2018rethinking,gale2019state,frankle2018lottery}. 
Other pruning algorithms prune the network at initialization \cite{lee2019signal} before performing a training phase.
Nevertheless, there exists a number of so-called \textit{data-free} approaches which do not require any data or fine-tuning of the pruned network, however usually resulting in lower pruning ratios. 
For instance, \cite{tanaka2020pruning} is a data-free pruning method with lower performances but still addresses the layer-collapse issue (where all the weights in a layer are set to $0$) by preserving the total synaptic saliency scores.
DREAM \cite{yin2020dreaming} proposes to use standard data-driven pruning algorithm over data-free generated data from DNN weight values.
Data-free approaches are usually outperformed by a significant margin on standard benchmarks by data-driven ones especially by methods who use extensive re-training.

In this work we propose to improve both data-free and similarity-based pruning. 
To do so, we introduce split for optimal redundancies removal in DNNs.

%
%
%
%
\section{DNN Weights Hashing : Method, Theory and Experiments}\label{sec:preliminaries}
%
%
\begin{figure*}[!t]
    \centering
    \includegraphics[width = \linewidth]{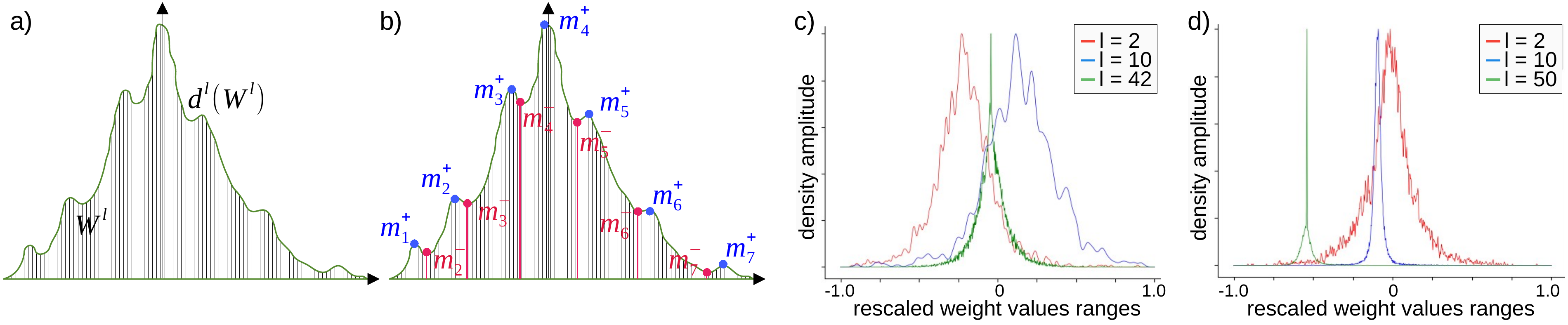}
    \caption{For each layer $l$, the hashing algorithm estimates the density $\mathbb{P}^l_w$ based on the weights $W^l$ using KDE (a). The estimated density $d^l$ is then evaluated on a discreet grid in order to obtain local minima $M_l^-$ and maxima $M^+_l$ (b) to obtain a partition and hashed values for the hashing function $\mathfrak{h}^l$. On the right side, we show examples of weight distributions for real layers from ResNet 56 on Cifar10 (c) and ResNet 50 on ImageNet (d).}
    \label{fig:hashing}
\end{figure*}
The proposed method works on the scalar weights of a trained DNN. 
These values are encoded using floating-point numbers: thus each parameter may have a value among approximately $4$ billion distinct possibilities. 
Therefore two neurons defined by their weight vectors have close to zero chance to be identical. 
In order to consider vector-level redundancies it is necessary to use an approximation in similarity-based pruning. 
In what follows, we consider a trained neural network $f$ with $L$ layers ${(f^l)}_{l\in \llbracket 1;L\rrbracket}$ with weights ${(W^l)}_{l\in \llbracket 1;L\rrbracket}$ and output dimensions ${(n^l)}_{l\in \llbracket 1;L\rrbracket}$. 
We will consider both scalar values $w$ and weight tensors $W$ through out our study.
We assume the weights to be sampled from a distribution $\mathbb{P}^l_w$.
%
%
\subsection{Hashing Protocol}\label{sec:hashing_protocol}
A naive approach to reduce the number of distinct weight values is quantization into int8 values using the operator introduced in \cite{krishnamoorthi2018quantizing}.
This operator transforms the continuous weight values $w$ uniformly into $256$ evenly spread values $\tilde w$ (which is our generic notation for discreet weights).
This will constitute our baseline.
\newline
Contrary to such uniform quantization scheme, we propose an adaptive transformation.
As illustrated on Fig. \ref{fig:hashing} (a and b), for every layer $l$ the kernel density estimation $d^l$ of $W^l$ is defined by 
\begin{equation}\label{eq:kde}
    d^l : \omega \mapsto \frac{1}{n^l \times n^{l-1} \Delta_l}\sum_{w \in W^l} K\left( \frac{\omega-w}{\Delta_l} \right)
\end{equation}
where $K$ is the Gaussian kernel of bandwidth $\Delta_l$. We extract the local minima $M^-_l = {(m_k^-)}_{k}$ and maxima $M^+_l = {(m_k^+)}_{k}$ from $d^l$.
The local minima $M^-_l = {(m_k^-)}_{k}$ define a partition of the support of $W^l$ such that we can define the hashing function $\mathfrak{h}^l$ as
\begin{equation}
    \mathfrak{h}^l : \omega \mapsto \sum_{m_k^+ \in M^+_l} m_k^+ \mathbbm{1}_{\{\omega \in \left[m_k^- ; m_{k+1}^- \right[\}}
\end{equation}
Then the hashed layer $\tilde f^l$ is defined as $\tilde w^l = \mathfrak{h}^l(w^l)$ and $\tilde W^l$ denotes the matrix of hashed weights.
In practice, we find that per-layer DNN weight values concentrate around a limited number of local modes as shown in Fig. \ref{fig:hashing} (c and d).
As a consequence, the proposed data-free adaptive hashing significantly reduces the number of distinct values among the weights, and introduces redundancies both at the vector and tensor level.
The bandwidth $\Delta_l$ affects the number of redundancies (the higher the bandwidth the higher the scalar redundancies) and the accuracy (the higher the bandwidth the lower the accuracy).
The proposed hashing method introduces an error from the original model. In what follows, we bound this error on the DNN logits in order to provide theoretical guarantees on the model accuracy preservation through hashing.
%
%
\subsection{Single Layer Preservation Through Hashing}
For simplicity, we propose to first study the case of a simple perceptron $f$ with weights $W \in \mathbb{R}^{n^{0}\times n^{1}}$ before extending the result to the whole network. 
We introduce the pseudo distance between the original weight values and their hashed version $|w - \tilde w|$ (note that this is a difference between scalars). 
We have the following upper bound:
\begin{equation}
    |w - \tilde w| \leq \min_{m\in M^+}\{m > w\} - \max_{m\in M^+}\{m < w\} 
\end{equation}
This follows from the fact that the hashing is based on a partition in segments of the support of $W$, \textit{i.e.} for all consecutive pair $m_{i},m_{i+1} \in M^+$ we have $\forall w \in  [m_{i};m_{i+1}]$, $\tilde w \in \{m_{i},m_{i+1}\}$. 
We can deduce an upper bound on the expected value of $|w - \tilde w|$
\begin{equation}\label{eq:A}
    \mathbb{E}[|w - \tilde w|] \leq \sum_{m_{i},m_{i+1}\in M^+ \cup \{m_{1}^-, m_{|M^-|}^-\}} \!\!\!\!\!\!\!\!(m_{i+1} - m_{i}) \int_{m_{i}}^{m_{i+1}} \mathbb{P}_w dw
\end{equation}
where $dw$ is the density of $W$ and $m_{1}^-, m_{|M^-|}^-$ are the minimum and maximum of $M^-$ respectively where $\mathbb{P}_w$ is the density of the weights.

We can compute another upper bound using two properties of the density estimation used. 
First, on each interval $[m_{i}^-,m_{i+1}^-]$ the function $d^l$ is mono-modal. 
Second, its variance can be estimated. We have:
\begin{equation}
    \mathbb{E} \left[ |w - \tilde w| \right] = \int_{W^{\min}}^{W^{\max}} |w - \tilde w| \mathbb{P}_wdw 
\end{equation}
where $W^{\min}$ and $W^{\max}$ are the minimum and maximum of $W$ respectively. 
This expression can be split in a sum using the partitioning used in the hashing
\begin{equation}
    \mathbb{E} \left[ |w - \tilde w| \right] = \sum_{i = 1}^{|M^+|}\int_{m_{i}^-}^{m_{i+1}^-} \left|w - m_{i}^+\right| \mathbb{P}_wdw
\end{equation}
where the $m_{i}^+$ and $m_{i}^-$ are the ordered elements of $M^+$ and $M^-$ respectively. 
We know from \cite{basu1997mean}, that the mean and mode of a unimodal distribution lie within $\sqrt{3}$ standard deviations of each other.
Because of the kernel used in KDE, we have a sum of Gaussian, thus we can use the following property.
The average absolute distance of a random sample to the mean of the distribution, under the Gaussian prior is $\sigma \sqrt{2/\pi}$. 
Using the triangular inequality we deduce that
\begin{equation}
    \mathbb{E} \left[ |w - \tilde w| \right] \approx \!\!\int\!\! |w - \tilde w| d(w)dw  \leq \max_{i \in \llbracket 1; |M^+|\rrbracket} \sigma_i \left(\sqrt{\frac{2}{\pi}}+\sqrt{3}\right)
\end{equation}
where $\sigma_i$ is the standard deviation of the distribution $d$ restricted to $[(M^-)_i;(M^-)_{i+1}]$. 
We can compute $\sigma_i$ using eq \ref{eq:kde}.
\begin{equation}
    \sigma_i^2 = \int_{(M^-)_i}^{(M^-)_{i+1}} \frac{1}{n^i\Delta} \sum_{n = 1}^{n^i} K\left(\frac{w -w_n}{\Delta}\right) w^2 dw - \mathbb{E}_d[X]^2
\end{equation}
Because $K$ is a Gaussian kernel we can deduce the value of $\sigma_i$ and update the formula for $\mathbb{E} \left[ |w - \tilde w| \right]$ the upper bound.
\begin{equation}\label{eq:B}
\mathbb{E} \left[ |w - \tilde w| \right] \leq \frac{\Delta}{\sqrt{2\pi}} \left(\sqrt{\frac{2}{\pi}}+\sqrt{3}\right)
\end{equation}
Now we have two upper bounds for $\mathbb{E}$. 
First $A$, from equation \ref{eq:A}, which based on the pseudo distance $\mathfrak{d}$ and the hashing function properties. 
Second $B$, from equation \ref{eq:B}, which is based on the kernel density estimation properties. We can combine these bounds to obtain
\begin{equation}\label{eq:hashing_1}
\mathbb{E} \left[ |w - \tilde w| \right] \leq \min\{A,B\} = u
\end{equation}
$u$ is our per-weight upper bound on the error made.
In practice both bounds $A$ and $B$ are relevant and used according to the situation.
This result can be extended to DNN with multiple layers.
%
%
\subsection{Multi-layer Preservation Through Hashing}
In order to generalize the previous upper bound to a feed forward CNN with $L$ layers, we first compute the upper bound for a layer $f^l$.
The upper bound $u_l$ measures the error on each weight values.
The weights are used in scalar products with $n^{l-1}\times w^l\times h^l\times n^l$ elements.
The average error behaves following the Central Limit Theorem, as detailed in Appendix \ref{sec:appendix_U_details}.
This adds a multiplicative term $\frac{1}{\sqrt{n^{l-1}w^lh^l}}$ to the expected error per layer.
Furthermore, we need to take into account the activation function.
Assuming a ReLU activation function, statistically, the average proportion of negative inputs is given by the CDF of a Gaussian distribution of parameters $\mu^l$ and $\sigma^l$.
These statistics are obtained from the batch normalization layers.
This adds a multiplicative term $\left(1 - \text{erf}\left(\frac{-\mu^l}{\sigma^l\sqrt{2}}\right)\right)$.
Therefore we get
\begin{equation}\label{eq:hash_per_layer}
    \mathbbm{E}[\|\tilde f^l - f^l\|] \leq \frac{u_l}{\sqrt{n^{l-1}w^lh^l}}\left(1 - \text{erf}\left(\frac{-\mu^l}{\sigma^l\sqrt{2}}\right)\right)  
\end{equation}
where $u_l$ is the layer-wise upper bound described in equation \ref{eq:hashing_1} .
We extend this result recursively across all layers.
This is detailed in section \ref{sec:appendix_U_details} of the appendix.
\begin{equation}\label{eq:hashing_2}
    U = \prod_{l=1}^{L} \left(\frac{u_l}{\sqrt{n^{l-1}w^lh^l}}  \left(  1  -  \text{erf} \left( \frac{-\mu^l}{\sigma^l\sqrt{2}}\right)   \right)\right) + \mu^{l} - \prod_{l=1}^{L} \mu^{l}
\end{equation}
The value of each $u_l$ is a linear function of the bandwidth $\Delta_l$.
Therefore $U$ is also a linear function of $\Delta_l$ which is very low in practice (see section \ref{sec:hashing_protocol}): thus the reason why, in practice, the hashing error is very low.
To assess this theoretical study, we still need to empirically validate the hashing protocol.
First, we validate that the upper bound $U$ provides practical data-free guarantees on the accuracy preservation.
Second, we provide empirical results on the percentage of removed weight values and accuracy drop from hashing.

\subsection{Upper Bound Error on the Hashing Error}
\begin{table*}[th]
\renewcommand{\arraystretch}{1.15}
\caption{
Empirical evaluation of the theoretical study on the expected error from weight values hashing. 
We evaluate the average of the logits ($E[\text{norm}]$) which serves the purpose of data-free evaluation of the hashing protocol by comparing its value to $U$ defined in equation \ref{eq:hashing_2}.
This value's tightness is verified by comparing it to the empirical measure $\mathbb{E}_X[\| f\|]$.
We also measure the value of the upper bound $U$ relatively to $\mathbbm{E}_X[\|\tilde f - f\|]$.
Then we estimate its tightness to the data-driven value using the proposed metric defined as the ratio $(U-\mathbbm{E}_X[\|\tilde f - f\|])/U$.
}
\label{tab:upper_bound}
\centering
\setlength\tabcolsep{3pt}
\begin{tabular}{|c||c|c|c|c||c|c||c|c|c|c|c||c|c|c|c|c|c|c|c|}
\hline
\multicolumn{20}{|c|}{\cellcolor{Gray} Cifar10}\\
\hline
Architecture & \multicolumn{4}{c||}{ResNet} & \multicolumn{2}{c||}{W. ResNet} & \multicolumn{5}{c||}{MobileNet v2} & \multicolumn{8}{c|}{EfficientNet}   \\
\hline
Model & 20 & 56 & 110 & 164 & 28-10 & 40-4 & 0.35 & 0.5 & 0.75 & 1 & 1.4 & B0 & B1 & B2 & B3 & B4 & B5 & B6 & B7 \\
\hline
\hline
U & 3.1 & 1.7 & 2.7 & 3.1 & 3.5 & 3.0 & 3.2 & 2.6 & 2.7 & 4.0 & 2.4 & 4.9 & 7.0 & 1.4 & 7.4 & 6.7 & 3.7 & 2.8 & 2.8 \\
\hline
\hline
$E[\text{norm}]$ & 27 & 26 & 26 & 26 & 22 & 22 & 23 & 24 & 25 & 24 & 26 & 25 & 20 & 21 & 20 & 19 & 19 & 22 & 23 \\
$\mathbb{E}_X[\| f\|]$ & 33 & 35 & 35 & 36 & 36 & 36 & 35 & 36 & 35 & 36 & 35 & 22 & 22 & 21 & 22 & 22 & 22 & 20 & 20 \\
\hline
\hline
Tightness to $\mathbbm{E}_X[\|\tilde f - f\|]$ & 7\% & 54\% & 44\% & 13\% & 64\% & 79\% & 22\% & 6\% & 61\% & 65\% & 28\% & 57\% & 52\% & 35\% & 21\% & 65\% & 34\% & 13\% & 17\% \\
\hline
\end{tabular}
\newline
\vspace*{0.4 cm}
\newline
\begin{tabular}{|c||c|c|c||c|c|c|c|c||c|c|c|c|c|c|c|c|}
\hline
\multicolumn{17}{|c|}{\cellcolor{Gray}ImageNet}\\
\hline
Architecture & \multicolumn{3}{c||}{ResNet} & \multicolumn{5}{c||}{MobileNet v2} & \multicolumn{8}{c|}{EfficientNet}   \\
\hline
Model & 50 & 101 & 152 & 0.35 & 0.5 & 0.75 & 1 & 1.4 & B0 & B1 & B2 & B3 & B4 & B5 & B6 & B7 \\
\hline
\hline
U & 0.01 & 0.01 & 0.01 & 0.04 & 0.05 & 0.04 & 0.08 & 0.03 & 0.01 & 0.01 & 0.01 & 0.01 & 0.01 & 0.01 & 0.01 & 0.01  \\
\hline
\hline
$E[\text{norm}]$ & 5.1 & 5.7 & 5.0 & 0.2 & 0.4 & 0.4 & 0.3 & 0.7 & 0.6 & 0.7 & 0.7 & 0.5 & 0.4 & 0.6 & 0.7 & 0.5 \\
$\mathbb{E}_X[\| f\|]$ & 5.8 & 5.5 & 5.8 & 0.4 & 0.5 & 0.5 & 0.6 & 0.6 & 0.6 & 0.7 & 0.6 & 0.6 & 0.6 & 0.7 & 0.7 & 0.7 \\
\hline
\hline
Tightness to $\mathbbm{E}_X[\|\tilde f - f\|]$ & 6\% & 1\% & 1\% & 37\% & 55\% & 38\% & 27\% & 28\% & 11\% & 2\% & 41\% & 13\% & 3\% & 1\% & 19\% & 9\% \\
\hline
\end{tabular}
\end{table*}

In this section we answer the following question: do we have a data-free way to predict whether or not the hashing will harm the accuracy of the DNN?
To tackle this issue, we propose a simple criterion based on the upper bound $U$ and a data-free estimation of the norm of the logits.
The information about datasets and models is provided in Appendix \ref{sec:appendix_datasets}. 
The models presented in the main paper are ConvNets and are standard in recent computer vision benchmarks, we also tested architectures which heavily rely on fully-connected layers in Section \ref{sec:appendix_fully_connected_layers}).
For implementation details please refer to Appendix \ref{sec:appendix_implem_details}.

\subsubsection{Data-Free Evaluation of the Hashing Protocol}
Assuming that we have the value of the expected norm of the logits $\mathbb{E}_X[\| f\|]$, we can derive the following criterion to decipher whether the hashing is detrimental to the network accuracy:
\begin{enumerate}
    \item if $\frac{U}{\mathbb{E}_X[\| f\|]} << 1$ then the modifications from hashing won't have a significant impact on the logits. Therefore their order is likely unchanged. 
    \item if $\frac{U}{\mathbb{E}_X[\| f\|]} \simeq 1$ then the modifications from hashing may have a significant impact on the predictions. As $U$ is an upper bound we can't conclude yet.
\end{enumerate}
The values of $U$ and $\mathbb{E}_X[\| f\|]$ are reported in the first and third rows of the TABLE \ref{tab:upper_bound}.
On Cifar10, the ratio never exceeds $1/10$ to the exception of EfficientNets B0-B4 where the ratio reaches $1/3$.
Considering the number of logits and fact that recent DNNs are over-confident, our algorithm leads us to conclude that these networks accuracy will be preserved.
Similarly, on ImageNet, the value of the ratio $\frac{U}{\mathbb{E}_X[\| f\|]}$ ranges from $1/570$ to $1/8$ which corresponds to the case $\frac{U}{\mathbb{E}_X[\| f\|]} << 1$.
The modifications from hashing won't have a significant impact on the predictions which matches our empirical validation from section \ref{sec:experiments_hashing}.
However the value of $\mathbb{E}_X[\| f\|]$ is obtained using data.

To make this evaluation data-free, we propose a data-free estimator of the expected norm of the logits, referred to as $E[\text{norm}]$.
To do this, we use the values of the weights of the batch normalization layers \cite{ioffe2015batch} as an estimate of the expected value of the shallowest layer and then use the last kernel with the linearity of the expectation to compute $E[\text{norm}]$.
Because we use $\mathbb{E}_X[\| f\|]$ as the denominator, we need the estimate to be as close as possible while satisfying $E[\text{norm}]\leq \mathbb{E}_X[\| f\|]$ in order not to have an over-confident criterion.
Empirically, this is verified as shown in TABLE \ref{tab:upper_bound}.
In practice, we also use the variance and its estimate in order to obtain a confidence interval.
The variances $V[\text{norm}]$ are equal to about a third of the expectations.
When comparing the real values to the estimates we get $\frac{\mathbbm{E}_X[\|\tilde f - f\|]}{\mathbb{E}_X[\| f\|]} \leq \frac{U}{E[\text{norm}]} < \frac{U}{E[\text{norm}] - V[\text{norm}]} << 1$.
In consequence, our algorithm provides a data-free way to ensure the preservation of the accuracy of a hashed DNN.

\subsubsection{Tightness of the Upper Bound}
Let's assume we have a DNN such that $\frac{U}{E[\text{norm}]} << 1$ is not considered satisfied, \textit{i.e.} we are in the second case of our criterion.
The question that remains is: can we still have $\frac{\mathbbm{E}_X[\|\tilde f - f\|]}{E[\text{norm}]} << 1$?
This is equivalent to ask if $U$ is tight to $\mathbbm{E}_X[\|\tilde f - f\|]$.
We define the tightness metric as the ratio $(U-\mathbbm{E}_X[\|\tilde f - f\|])/U$ and report its values in TABLE \ref{tab:upper_bound}.
Because the theoretical value $U$ is an upper bound on $\mathbbm{E}_X[\|\tilde f - f\|]$ the value of the previous ratio necessarily lies in $[0;1]$ (the closer to $0$ the tighter the upper bound).
We observe values ranging from $7\%$ to $80\%$ with most values under $50\%$ on models trained for Cifar10 which indicates that the proposed upper bound is relatively tight.
However this difficulty appears to be solved on ImageNet as almost no network exceed $40\%$ and many values are below $10\%$.
The most remarkable results appear on EfficientNets on ImageNet and are stable across all network sizes.
This confirms that relatively to $\mathbbm{E}_X[\|\tilde f - f\|]$ the upper bound $U$ is tight.
In consequence, when we don't have the condition $\frac{U}{E[\text{norm}]} << 1$ we can assert with relative confidence that the hashing significantly modified the predictive function

Stemming on these results, we have an efficient and verifiable, adaptive data-free method for scalar redundancies extraction.
We leverage these redundancies via pruning.

%
%
\subsection{Weight Values Redundancy from Hashing}\label{sec:experiments_hashing}
\begin{table*}[!th]
\renewcommand{\arraystretch}{1.15}
\caption{Ablation results of the compression due to hashing defined as the percentage of removed weight values (dubbed \% reduction). We also report the accuracy drop on the test set. 
The uniform and adaptive hashing both remove an important proportion of weight values.
The adaptive hashing systematically preserves the accuracy contrary to the uniform approach.
}
\label{tab:hashing_ratio}
\centering
\setlength\tabcolsep{2.75pt}
\begin{tabular}{|c|c||c|c|c|c||c|c||c|c|c|c|c||c|c|c|c|c|c|c|c|}
\hline
\multicolumn{21}{|c|}{\cellcolor{Gray} Cifar10}\\
\hline
\multicolumn{2}{|c||}{Architecture} & \multicolumn{4}{c||}{ResNet} & \multicolumn{2}{c||}{W. ResNet} & \multicolumn{5}{c||}{MobileNet v2} & \multicolumn{8}{c|}{EfficientNet}   \\
\hline
\multicolumn{2}{|c||}{Model} & 20 & 56 & 110 & 164 & 28-10 & 40-4 & 0.35 & 0.5 & 0.75 & 1 & 1.4 & B0 & B1 & B2 & B3 & B4 & B5 & B6 & B7 \\
\hline
\hline
\multirow{2}{*}{\% Reduction} & uniform & 98.0 & 98.4 & 98.5 & 98.9 & 99.7 & 99.8 & \underline{98.4} & \underline{98.4} & 98.4 & 98.6 & 99.0 & 98.5 & 98.5 & 98.6 & 98.7 & 98.9 & 99.0 & 99.1 & 99.2 \\
 & \textbf{adaptive} & \underline{98.9} & \underline{99.0} & \underline{99.1} & \underline{99.1} & \underline{99.9} & \underline{99.9} & 96.7 & 98.0 & \underline{98.9} & \underline{99.4} & \underline{99.6} & \underline{99.4} & \underline{99.5} & \underline{99.5} & \underline{99.6} & \underline{99.7} & \underline{99.8} & \underline{99.8} & \underline{99.9} \\
\hline
\multirow{2}{*}{\% Accuracy drop} & uniform & 0.58 & 0.45 & 0.87 & 0.89 & 0.72 & 0.68 & \textbf{0.00} & 0.01 & 0.10 & 0.00 & 0.00 & 0.00 & 0.01 & 0.03 & \textbf{0.00} & 0.01 & \textbf{0.00} & \textbf{0.00} & \textbf{0.00} \\
 & \textbf{adaptive} & \textbf{0.07} & \textbf{-0.03} & \textbf{0.10} & \textbf{0.01} & \textbf{-0.02} & \textbf{0.00} & \textbf{0.00} & \textbf{0.00} & \textbf{0.00} & \textbf{-0.01} & \textbf{-0.04} & \textbf{-0.02} & \textbf{0.00} & \textbf{0.10} & \textbf{0.00} & \textbf{0.00} & \textbf{0.00} & 0.01 & \textbf{0.00} \\
\hline
\end{tabular}
\newline
\vspace*{0.4 cm}
\newline
\begin{tabular}{|c|c||c|c|c||c|c|c|c|c||c|c|c|c|c|c|c|c|}
\hline
\multicolumn{18}{|c|}{\cellcolor{Gray}ImageNet}\\
\hline
\multicolumn{2}{|c||}{Architecture} & \multicolumn{3}{c||}{ResNet} & \multicolumn{5}{c||}{MobileNet v2} & \multicolumn{8}{c|}{EfficientNet}   \\
\hline
\multicolumn{2}{|c||}{Model} & 50 & 101 & 152 & 0.35 & 0.5 & 0.75 & 1 & 1.4 & B0 & B1 & B2 & B3 & B4 & B5 & B6 & B7 \\
\hline
\hline
\multirow{2}{*}{\% Reduction} & uniform & 99.6 & 99.5 & 99.5 & 21.2 & 26.4 & 27.7 & 24.7 & 22.0 & \underline{98.5} & \underline{98.5} & \underline{98.6} & \underline{98.7} & \underline{98.9} & \underline{99.0} & \underline{99.1} & \underline{99.2}  \\
& \textbf{adaptive} & \underline{99.9} & \underline{99.9} & \underline{99.9} & \underline{22.0} & \underline{32.9} & \underline{49.1} & \underline{61.3} & \underline{97.4} & 72.3 & 80.1 & 81.4 & 84.3 & 87.8 & 97.3 & 97.5 & 97.8 \\
\hline
\multirow{2}{*}{\% Accuracy drop} & uniform  & \textbf{0.00} & 0.01 & \textbf{0.00} & 3.96 & 6.90 & 7.71 & 5.79 & 8.19 & 3.36 & 14.32 & 6.69 & 4.91 & 3.85 & 2.30 & 4.90 & 3.30 \\
& \textbf{adaptive} & \textbf{0.00} & \textbf{0.00} & \textbf{0.00} & \textbf{0.25} & \textbf{0.56} & \textbf{0.30} & \textbf{0.60} & \textbf{0.10} & \textbf{0.44} & \textbf{0.36} & \textbf{0.10} & \textbf{0.25} & \textbf{0.00} & \textbf{0.00} & \textbf{0.00} & \textbf{0.01} \\
\hline
\end{tabular}
\end{table*}
In order to validate the effectiveness of the proposed hashing for exhibiting redundancies in DNNs, We define the hashing compression ratio as the percentage of removed distinct weight values and report this metric in Table \ref{tab:hashing_ratio} as well as the accuracy drop.

First, we compared the naive uniform baseline (int8 quantization) with the proposed adaptive method: on Cifar10, we observe that the proposed adaptive hashing achieves better figures both in terms of percentage of removed weight values and accuracy.
On all networks, the hashing introduces negligible accuracy drops while reaching more than 99\% compression in term of distinct values.
On ImageNet, the results echo the Cifar10 performances, except on EfficientNets B0-B4, where uniform quantization achieves higher compression ratios at the expense of significant (up 14.32\%) accuracy drop.
Meanwhile, our adaptive hashing still achieves important compression ratios on these networks with very low (up to 0.44\%) loss in accuracy.
Overall, on the two datasets, our adaptive hashing allows to greatly reduce the number of distinct weight values.
The performance on Cifar10 are outstanding as we reach 99\% parameters reduction on almost all architectures.
We observe a form of invariance to the network's depth as the performance is almost identical for ResNet 20, 56 and 110.
We also observe robustness to architectural changes such as depthwise convolution as MobileNets and ResNets are compressed with a similar intensity on Cifar10. 
On ImageNet networks, we observe impressive results on large networks, e.g. ResNets with a compression ratio of $99.9\%$.
We also evaluated the method over already compact architectures (MobileNets and EfficientNets) and adapted the protocol by considering only the layers that would be a target for pruning, \textit{i.e.} we didn't hash the depth-wise convolutional layers as hashing these had a dramatic impact on the networks accuracy. 
Furthermore, the predictions of these models (MobileNet) have a lower range and especially the difference between the two highest logits is much narrower (especially on ImageNet): therefore these models may be much more sensible to weight values modification.
Despite this, the method still removes a considerable number of distinct values for these small networks, e.g. MobileNet V2 with width multiplier 0.35 loses 22\% of its unique values while MobileNet V2 with width multiplier 1.4 loses 97\%.

Overall the hashing mechanism removes more than $90\%$ weight values, to the exception of MobileNet models on ImageNet which proves the interest of the proposed adaptive hashing as a standalone memory footprint reduction technique (e.g. as compared with the naive baseline).
Last but not least, it allows to introduce redundancies in DNNs, a fact that is leveraged in our DNN pruning method.

%
%
%
%
\section{Removal of Redundant Operations in DNNs}\label{sec:preliminaries_method}
From $W^l\in \mathbb{R}^{w\times h\times n^{l-1} \times n^l}$ to $\tilde W^l$, its hashed counter part, the number of distinct weight values is reduced. 
Thus, the hashing step exhibits scalar redundancies. 
In order to exploit these redundancies we propose to leverage two pruning mechanisms that exploit structural redundancies.
%
%
\begin{figure}[t]
    \centering
    \includegraphics[width = \linewidth]{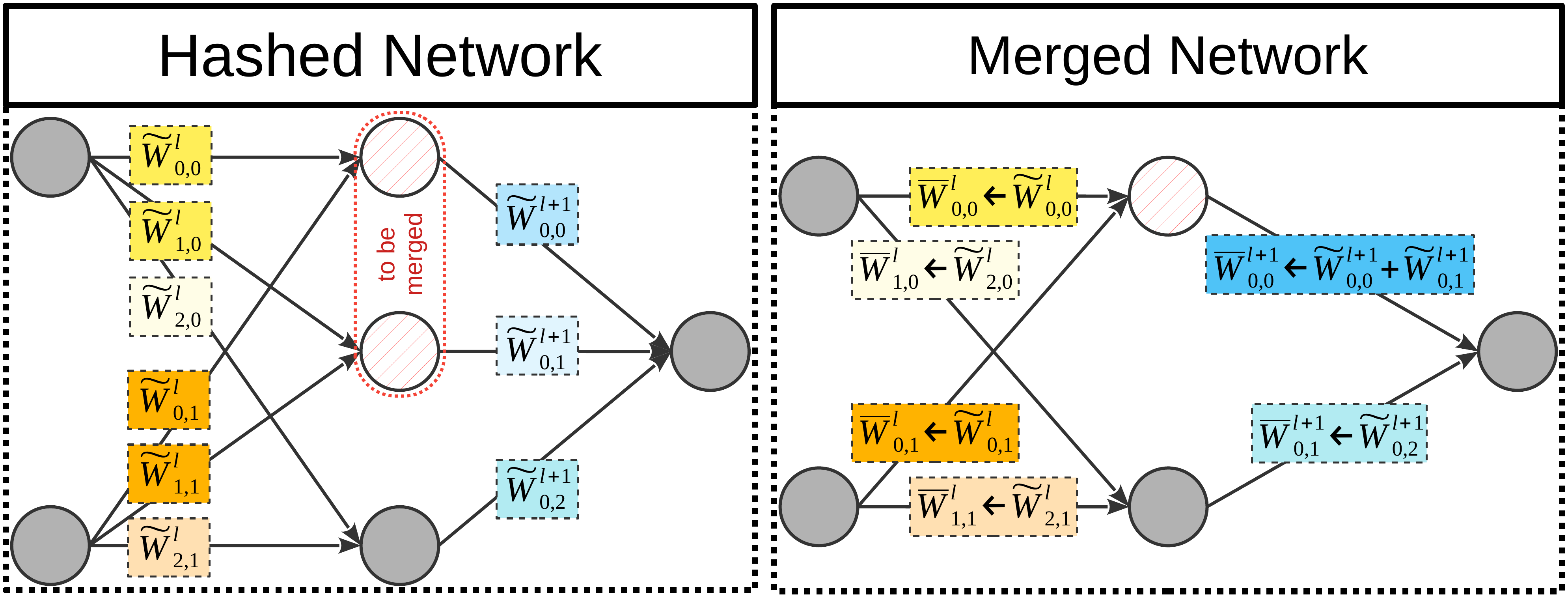}
    \caption{Neuron merging in the case of a fully-connected layer $l$ with weights $\tilde W^l_{i,j}$. Similarity between weight values is displayed by color, e.g. $\tilde W^l_{0,0} = \tilde W^l_{1,0}$. The merged network weights $\bar W$ are obtained by merging the first 2 neurons of layer $l$ and updating the consecutive layer by simply summing the corresponding weights.}
    \label{fig:merge}
\end{figure}
\subsection{Output-wise Neuron Merging}
For every layer $l$ the redundancy of information computed by two filters (or neurons in the case of fully connected layers) is defined by the euclidean distance between their respective weight values.
Formally, for $\tilde W^l \in {M^+_l}^{w \times h \times n^{l-1} \times n^l}$ we define a similarity-based filter clustering. 
Two filters are in the same cluster if and only if their distance is below the $\alpha^l\%$ smallest non-zero distance of the current layer. 
In Fig \ref{fig:merge}, this is illustrated in the first layer with the first two neurons being identical as indicated by the colormap. 
We detail the selection of the layer-wise hyperparameter $\alpha^l$ in Appendix \ref{sec:appendix_alpha}. 
Let's assume that we have $\bar n^l < n^l$ distinct neurons, \textit{i.e.} $\bar n^l$ distinct rows in $\tilde W^l$.  
Let $ \bar W^l$ be the sub-matrix  of $\tilde W^l$ containing all the distinct rows of $\tilde W^l$ only once and $\bar W^{l+1}$ the matrix such that all columns from $\tilde W^{l+1}$ that were applied to identical neurons of $\tilde W^l$ are summed. 
Then, for each output dimension $i$ we have:
\begin{equation}\label{eq:merge}
\begin{aligned}
    \Big(\tilde f^{l+1} (z)\Big)_{i} & =\sum_j^{n^l} \tilde W^{l+1}_{i,j} \sigma\left(\sum_k^{n^l} \tilde W^l_{j,k} z_k \right)\\
                               & \simeq \sum_j^{\bar n^l} \bar W^{l+1}_{i,j} \sigma\left( \sum_k^{\bar n^l} \bar W^l_{j,k} z_k \right) = \Big(\bar f^{l+1} (z)\Big)_{i}
\end{aligned}
\end{equation}
We obtain the clustered weights $\bar W^l \in \mathbb{R}^{w \times h \times n^{l-1} \times \bar n^l}$ with $\bar n^l \leq n^l$. 
In order to preserve the coherence of the predictive function, we update the consecutive layers by adding the weights of the corresponding input dimensions as illustrated in Fig \ref{fig:merge}. 
In the special case $\alpha^l = 0$ we only merge strictly identical neurons and we have $\bar f^l = \tilde f^l$ and only have equalities in equation \ref{eq:merge}.
We introduce the value $\alpha^*$ for $\alpha$ which corresponds to the highest value of $\alpha$ such that the accuracy is preserved, \textit{i.e.} the function is changed but the precision on a validation set is preserved.
The proposed neuron merging algorithm provides a simple way to remove redundant weights in DNNs.
In addition to that, in what follows, we introduce an input-wise decomposition for pruning.
\subsection{Input-wise Neuron Splitting}
%
%
\begin{figure}[t]
    \centering
    \includegraphics[width = \linewidth]{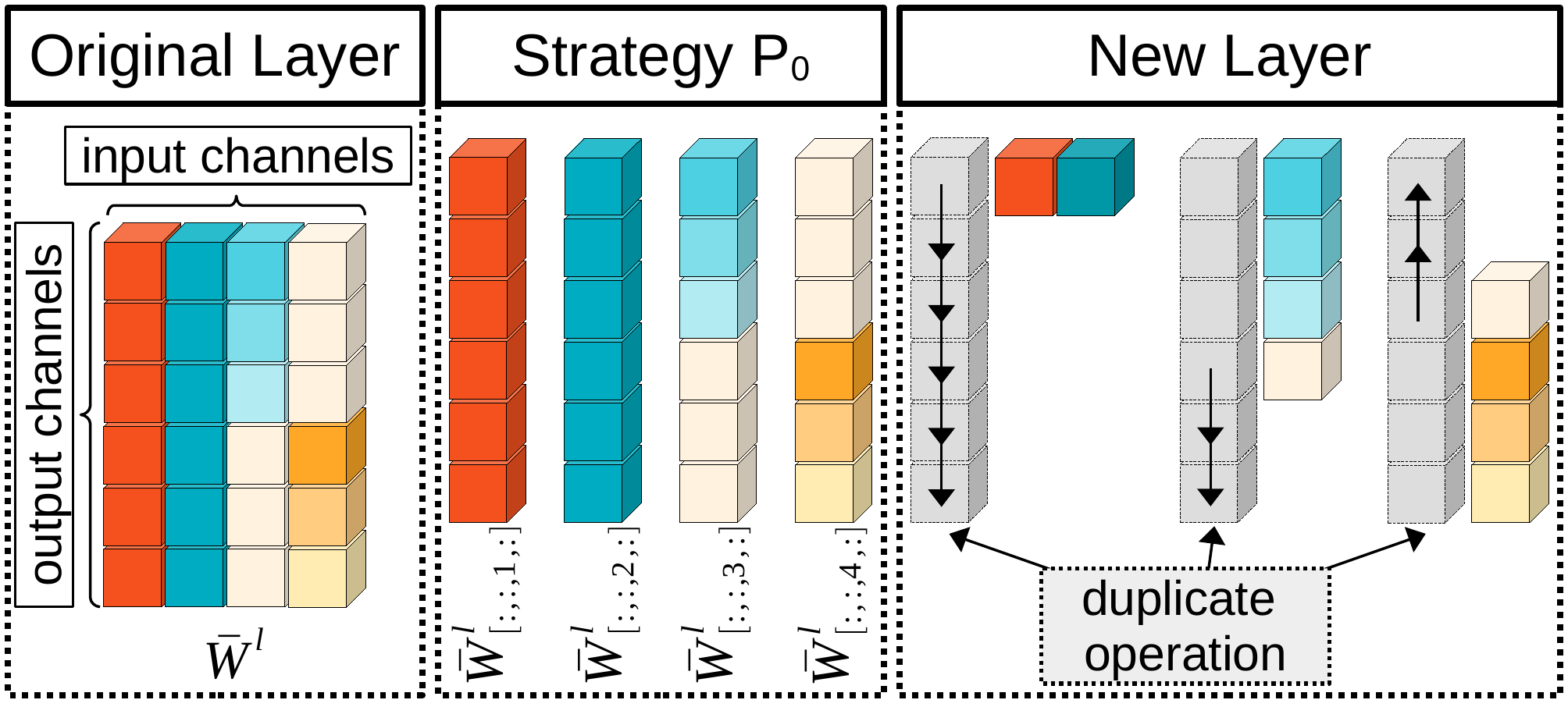}
    \caption{Splitting strategy for a layer with weights $\bar W^l \in \mathbb{R}^{4\times 6}$ where colors indicate weight values. We apply the singleton strategy by splitting the layer w.r.t. its input. We then collapse identical subsets of each components of the partition and define the corresponding duplication functions.}
    \label{fig:split_layer}
\end{figure}
We propose a novel way to split layers in order to reveal local structured redundancies. 
We assume the mathematical operations (e.g. multiplications, additions) to be bottlenecks when compared to memory allocation. 
This condition is motivated in Appendix \ref{sec:appendix_multi_dupli}. 
When two identical operations (as represented by redundant weight values) are performed on the same inputs, the outputs will systematically be identical.
Based on this observation, in this section we describe a splitting scheme to reduce the number of operations.\newline
Formally, let's consider a convolutional layer $l$ with weights $\bar W^l \in \mathbb{R}^{w\times h\times n^{l-1} \times n^l}$ and activation function $\sigma$. 
The case of convolutional layers extends the case fully connected layers by setting $w = h = 1$.
Let $\mathcal{C}^l$ be a partition of the set of inputs $\llbracket 1;n^{l-1} \rrbracket$.
We define the partial layers $\bar f_{C_i}^l$, for each element $(C_i^l)$ of $\mathcal{C}^l$ (\textit{i.e.} for each sub-set of input dimensions), by
\begin{equation}
    \left(\bar f_{C_i^l}^l\right)_j : x \mapsto \sum_{c \in C_i^l} \bar W^l_{[:,:,c,j]} * x
\end{equation}
Each operation performed on an input channel is defined by the corresponding weight values $\bar W^l_{[:,:,c,:]}$. 
In the proposed splitting method, we collapse identical sub neurons by removing redundant operations. 
To preserve the output values and dimensions we add a duplication function $d^l_C$ which allocates the results to their corresponding position in the original layer $\bar f^l$. 
Let's assume that in the sub set $C_i^l$ of the partition $\mathcal{C}^l$ we have $j$ and $j'$ such that 
\begin{equation}
  \bar  W^l_{[:,:,c,j]} = \bar W^l_{[:,:,c,j']}
\end{equation}
Then the corresponding 2D convolutions are identical and applied to the same input, \textit{i.e.} $\left(\bar f_{C_i^l}^l\right)_j = \left(\bar f_{C_i^l}^l\right)_{j'}$. 
Under this condition, split will collapse the two operations such that the resulting layer $\hat f^l$ won't contain the weights from $\left(\bar f_{C_i^l}^l\right)_{j'}$ and will instead duplicate the results from $\left(\hat f_{C_i^l}^l\right)_j$ at $j'$ position in the output.
The resulting partial layer $\hat f_{C_i}^l$ with weights $\hat W^l$ is a concatenation on the input channel axis of the corresponding unique subset of weight values and we have
\begin{equation}
     \bar f^l = \sigma \left( \sum_{C \in \mathcal{C}^l} d^l_{C}\left(\hat f_{C}^l\right) \right) = \hat f^l
\end{equation}
Note that this decomposition doesn't modify the hashed predictive function.
The remaining problem is how to find the clusters $(C_i^l)$.
The brute force approach to find the strategy of clustering that offers the lowest number of remaining operations (which defines optimality here) is intractable.
It would require to search among the $B_{n^{l-1}}$ possible partitions, where $B_{n^{l-1}}$ is the Bell's number \cite{barsky2004nombres} for a set of $n^{l-1}$ distinct elements which is defined recursively as 
\begin{equation}
    B_{n^{l-1}} = \sum_{k=1}^{n^{l-1}-1} \begin{pmatrix} n^{l-1}-1 \\k \end{pmatrix} B_k
\end{equation}
To cope with the problem, there is an optimal strategy with regards to the number of multiplications which can be found in a constant time. 
The strategy consists in taking the partition in singletons of $\llbracket 1;n^{l-1} \rrbracket$, as illustrated in Fig \ref{fig:split_layer}. 
We also provide proof that this protocol is optimal on $1\times1$ convolutions under minimal assumptions.
%
%
\subsection{Optimality of the Singleton Decomposition}
We provide proof that the proposed decomposition is optimal according to the criterion of the number of operations (multiplications in the case of fully connected layers or convolutions in the case convolutional layers). 
Let's consider the weights tensor $W^l \in M^{n^{l-1} \times n^l}$, where $M$ is $\mathbb{R}$ for fully connected layers and $\mathbb{R}^{w\times h}$ for convolutional layers. 
\begin{lemma}\label{thm:P_0}
The partitioning strategy that minimizes the number of operations performed by the layer is the decomposition in singleton.
\end{lemma}
\begin{proof}
We start by showing that the problem has at least one optimal solution. 
For any partition $\mathcal{C}^l = {(C_i^l)}_{i\in I}$ of $\llbracket 1; n^{l-1} \rrbracket$ the resulting collapsed layer $\hat f_{\mathcal{C}^l}^l$ performs a finite number of mathematical operations $n_{\mathcal{C}^l}$. 
Then we can define the objective function
\begin{equation}
    \begin{matrix}
        G^l : & \text{Part}(\llbracket 1;n^{l-1} \rrbracket) & \rightarrow & \llbracket 1;n^{l-1}n^l \rrbracket\\
            & \mathcal{C}^l & \mapsto & n_{\mathcal{C}^l}
    \end{matrix}
\end{equation}
where $\text{Part}(\llbracket 1;n^{l-1} \rrbracket)$ is the set of all possible partitions of $\llbracket 1;n^{l-1} \rrbracket$. 
We obtain the upper bound $n^{l-1}n^l$ on the number of operations from the property of the split decomposition which can only diminish that number. 
Because of that, $G^l$ is a function between two finite sets and thus admits at least one minimum. 
Therefore the minimization problem has a solution.\newline
Now let's assume a partition $\mathcal{C}^l$ of $\llbracket 1;n^{l-1} \rrbracket$ contains a set $C$. 
Then the split condition to fuse two sets of operations is: \textit{there exists $j$ and $j'$ such that $W^l_{[c_i,j]} = W^l_{[c_i,j']}$ for all $c_i\in C$}.
This is clearly a special case of the proposed partition. 
We recall that the proposed partition induces the following fuse condition: there exists $j$ and $j'$ such that $W^l_{[c,j]} = W^l_{[c,j']}$ for some $c \in \llbracket 1;n^{l-1} \rrbracket$. 
Thus we have $n_{\mathcal{C}^l} \geq n_{P^l}$ where $P$ refers to the proposed strategy. 
\end{proof}

\subsection{Expected Pruning}\label{sec:preliminaries_birthday}
%
%
\begin{figure}[t]
    \centering
    \includegraphics[width = 0.9\linewidth]{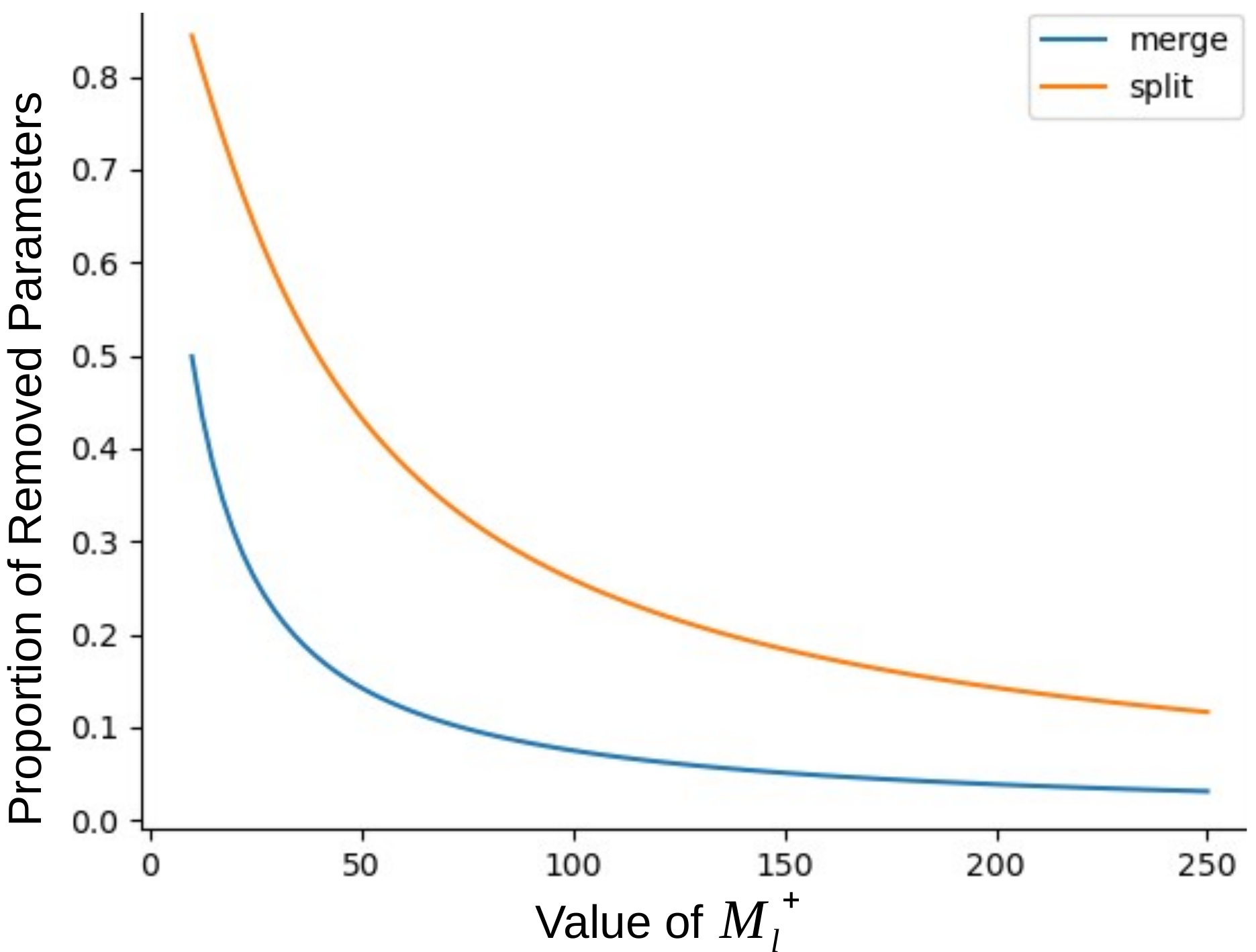}
    \caption{Given a layer $l$ with weights $W^l \in \mathbb{R}^{3\times 3\times 32 \times 128}$ sampled uniformly, we apply hashing such that we keep $M^+_l$ values. We plot the evolution of $\mathbb{E}_m$ (blue) and $\mathbb{E}_s$ (orange) for different values of $M^+_l$ using the formula from equations \ref{eq:uniform_E_m} and \ref{eq:uniform_E_s}. We observe the effect of the hashing on the expected pruning.
    }
    \label{fig:Es_Em}
\end{figure}
We showed that we can remove redundant operations \textit{via} output-wise merging of similar neurons as well as input-wise splitting.
In this section, we conduct a probabilistic study of the relationship between reducing the number of unique weight values and the resulting pruning ratio in DNNs. 
This can be cast as a case of the generalized Birthday Problem. 
We assume that each weight is sampled independently from a unique distribution.
First, assuming the least favorable prior (uniform) on this distribution, we want to compute the expected pruning factor from the merging step as well as the splitting step noted $\mathbb{E}_m$ and $\mathbb{E}_s$ respectively.\newline
First we recall the probability $\mathbb{P}_n^K(k)$ to get $k$ distinct values from $K$ possible ones with a sampling size $n$.
The probability that each of $n$ samples belongs to those $k$ values is $(k/K)^n$. 
However, this also includes cases where some of the $k$ numbers were not chosen. 
The inclusion-exclusion rule says that the probability of drawing all of those $k$ values is $(-1)^i \begin{pmatrix} k \\ i \end{pmatrix} \left(\frac{k-i}{K}\right)^n$.
As a consequence, we have:
\begin{equation}
    \mathbb{P}_n^K(k) = \begin{pmatrix} K \\ k \end{pmatrix} \sum_{i=0}^{k} (-1)^i \begin{pmatrix} k \\ i \end{pmatrix} \left(\frac{k-i}{K}\right)^n
\end{equation}
from this expression we can compute $\mathbb{E}_m$ and replace $K = whn^{l-1}M^+_l$ and $n = n^l$. 
This result is obtained by developing the standard definition of the expected value of a discreet variable.
For $\mathbb{E}=\mathbb{E}_m$ or $\mathbb{E}_s$ we have:
\begin{equation}
    \mathbb{E} = \sum_{k=0}^K k \mathbb{P}_n^K(k) = K\left(1 - \left( 1 - \frac{1}{K}\right)^n\right)
\end{equation}
it follows that the expected pruning ratio of the merging step is:
\begin{equation}\label{eq:uniform_E_m}
    \mathbb{E}_m = 1-\frac{whn^{l-1}M^+_l \left(1-\left(1-\frac{1}{ whn^{l-1}M^+_l}\right)^{n^l}\right)}{n^l}
\end{equation}
We illustrate in Fig \ref{fig:Es_Em} the variation of $\mathbb{E}_m$ as a function of the proportion of the unique values among hashed weights over the original number of distinct weights, \textit{i.e.} $ M^+_l/ (whn^{l-1}n^l)$ for different values of $n^{l-1}$.
These results suggest that the larger the input dimension $n^{l-1}$ the lower the number of redundancies. 
The value of $\mathbb{E}_s$ can be computed similarly with different values of $K$ and $n$. 
We obtain
\begin{equation}\label{eq:uniform_E_s}
    \mathbb{E}_s = 1-\frac{whM^+_l \left(1-\left(1-\frac{1}{whM^+_l}\right)^{n^l}\right)}{n^l}
\end{equation}
Similarly to $\mathbb{E}_m$, we illustrate $ \mathbb{E}_s$ in Fig \ref{fig:Es_Em}. 
From these results follows 
\begin{lemma}
Under the uniform prior, for any layer with weights $W^l\in \mathbb{R}^{w\times h\times n^{l-1} \times n^l}$ and hashed weights $\tilde W^l$ we have
\begin{equation}
    \mathbb{E}_s > \mathbb{E}_m
\end{equation}
\end{lemma}
This result extends to any prior on the distribution of single weight values of $\tilde W^l$.
\subsection{Prior-less Generalized Birthday Problem}\label{sec:generalized_bday}
In order to extend the previous result we need a preliminary result which extends the birthday problem to non-uniform sampling. 
In this section we note $E = \llbracket 1;m \rrbracket$ the sampling space, ${(X_j)}_{i\llbracket 1; n\rrbracket}$ the i.i.d. variables sampled from law $\mathcal{L}$ which satisfy
\begin{equation}\label{eq:hypothesis_Lm_Ls}
    \forall i\in E, \quad p_i = \mathbb{P}_\mathcal{L} (X_j = i) > 0
\end{equation}
We introduce the variables ${(Y_i)}_{i\in E}$ the number of samples $X_j = i$, \textit{i.e.} $Y_i = \sum_{j=1}^n \mathbbm{1}_{X_j = i}$. 
In our case, we compute the expected value of the number $V$ of distinct values in the sample, \textit{i.e.}
\begin{equation}\label{eq:prior_less_bday_}
    V = \sum_{i=1}^m \mathbbm{1}_{Y_i > 0} = m - \sum_{i=1}^{m}\mathbbm{1}_{Y_i = 0}
\end{equation}
We consider $i$ a value belonging to a subset $K\subset E$.
We note $B_i$ the event where the value $i$ is not sampled:
\begin{equation}\label{eq:prior_less_bday}
\begin{cases}
    \mathbb{P}_\mathcal{L}\left(\underset{i\in K}{\bigcap} B_i\right) = \left(1- \sum_{i\in K} p_i\right)^n \\
    \mathbb{E}_\mathcal{L}[V] = \sum_{v=1}^m v \sum_{K\in A^m_{m-v}} \left(1- \sum_{i\in K} p_i\right)^n
\end{cases}
\end{equation}
where $A^m_{m-v}$ is the set of arrangements of $E$. 
\begin{lemma}\label{thm:all_prior}
Under the priors $\mathcal{L}^m$ on $x = (x_1,\dots,x_{n^{l-1}}) \in {M^l_+}^{w\times h \times n^{l-1}}$ and $\mathcal{L}^s$ on $x_j \in {M^l_+}^{w\times h}$, for any layer with hashed weights $\tilde W^l\in \mathbb{R}^{w\times h\times n^{l-1} \times n^l}$, such that $n^l = n^{l-1}$ we have
\begin{equation}\label{eq:thm_guarantees++}
    \mathbb{E}_s > \mathbb{E}_m
\end{equation}
\end{lemma}
\begin{proof}
The expected pruning factor for the splitting step is higher than for the merging step if and only if $\mathbb{E}_{\mathcal{L}^s}[V] < \mathbb{E}_{\mathcal{L}^m}[V]$, with $V$ defined in equation \ref{eq:prior_less_bday_}, that is to say that the expected number of remaining neurons is lower under the prior $\mathcal{L}^s$.
Let's develop $\mathbb{E}_{\mathcal{L}^m}[V]$, following eq \ref{eq:prior_less_bday} we simply replace the notations with $m = n^l$ and the $p_i = \mathbb{P}_{\mathcal{L}^m}$. 
Thus we get
\begin{equation}\label{eq:dev_of_E_lm}
    \mathbb{E}_{\mathcal{L}^m}[V] = \sum_{v=1}^{m_m} v \!\!\sum_{K\in A^{m_m}_{m_m-v}}\!\! \left(\!\! 1- \sum_{i\in K} \mathbb{P}_{\mathcal{L}^m}(X = (x_1,...,x_{n^{l-1}})\!\!\right)^{n^l}
\end{equation}
where $m_m = |M^+_l|whn^{l-1}$. In the case of $\mathcal{L}^s$, like in the case of $\mathcal{L}^m$ we have $n^l$ samples, thus we replace $n = n^{l}$ and $p_i = \mathbb{P}_{\mathcal{L}^s}$ to obtain
\begin{equation}\label{eq:dev_of_E_ls}
    \mathbb{E}_{\mathcal{L}^s}[V] = \sum_{v=1}^{m_s} v \sum_{K\in A^{m_s}_{m_s-v}} \left(1- \sum_{i\in K} \mathbb{P}_{\mathcal{L}^s}(X = x_i)\right)^{n^{l}}
\end{equation}
where $m_s = |M^+_l|wh$.
However by equation \ref{eq:hypothesis_Lm_Ls}, it follows that
\begin{equation}
    \mathbb{P}_{\mathcal{L}^m}(X = x) = \prod_{j = 1}^{n^{l-1}} \mathbb{P}_{\mathcal{L}^s}(X_j = x_j) > 0
\end{equation}
and $\mathbb{P}_{\mathcal{L}^m}(x) < \mathbb{P}_{\mathcal{L}^s}(x')$ when $x'$ is a coordinate of $x$.
In consequence, if we compare equation \ref{eq:dev_of_E_lm} and \ref{eq:dev_of_E_ls}, we have a larger sum of larger terms.
Therefore, we get the desired result.
\end{proof}
\begin{figure}[t]
    \centering
    \includegraphics[width = \linewidth]{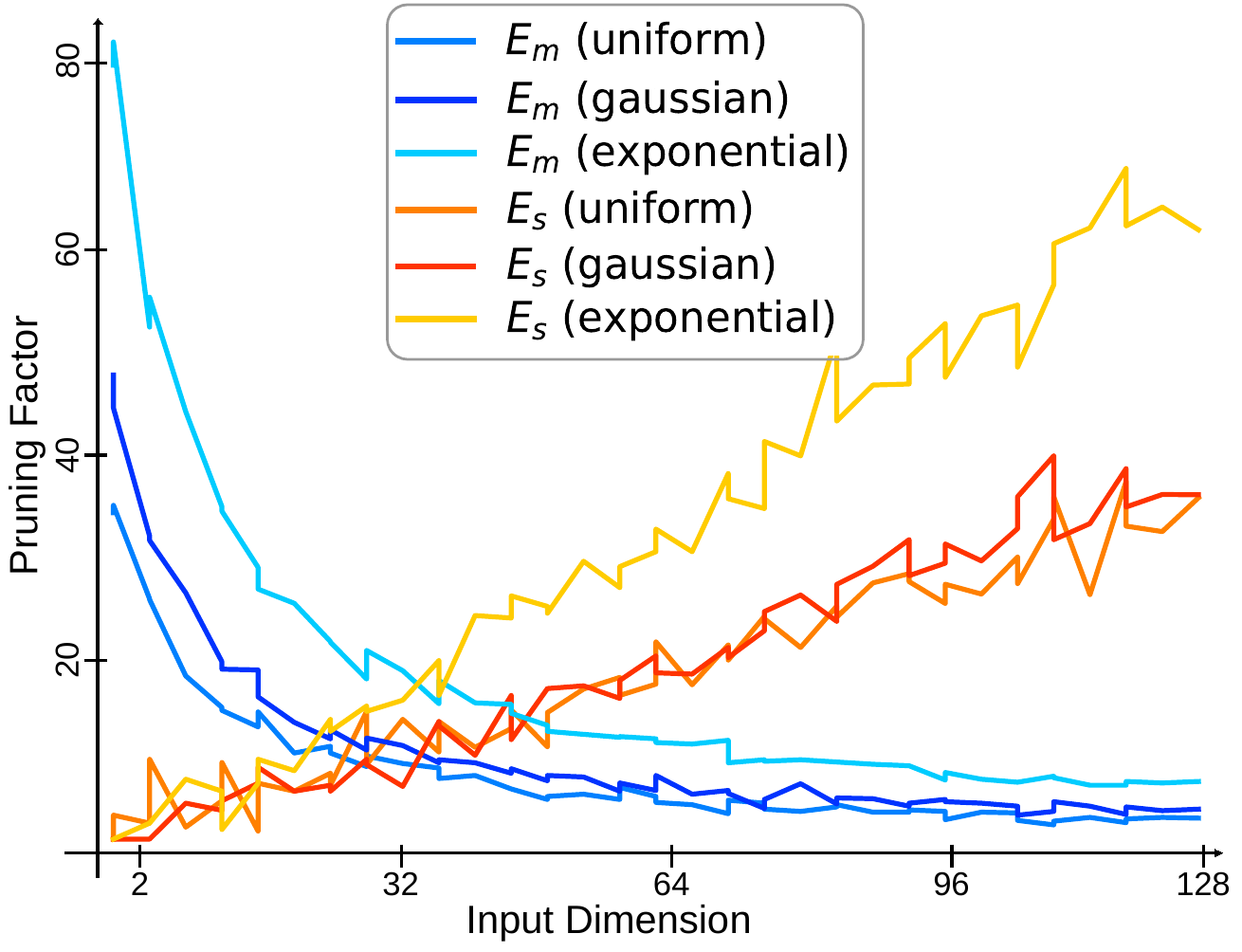}
    \caption{Given a layer $l$ with weights $W^l \in \mathbb{R}^{3\times3\times n^{l-1}\times n^l}$ we fix the value of the product $n^{l-1} \times n^l = 64^2$ and we plot the empirical values of the pruning ratios $\mathbb{E}_m$ (nuances of blue) and $\mathbb{E}_s$ (nuances of orange) for different values of $n^{l-1}$ the input dimension. The considered priors are the discreet Gaussian, the exponential distribution and the uniform distribution. We observe the complementarity of the merging and splitting.
    }
    \label{fig:Es_Em_all_priors}
\end{figure}
We propose in Fig \ref{fig:Es_Em_all_priors} an extension of this result for different configurations of $n^{l-1}$ and $n^l$. 
We tested the Gaussian, exponential and uniform priors, in the case of $3\times 3$ convolutional layer with weights $W\in \mathbb{R}^{3\times 3 \times n^{l-1}\times n^l}$ and $M^+ = 100$ constant. 
We vary the input dimension $n^{\text{in}}$ linearly from $2$ to $128$ with $n^l n^{l-1} = 64^2$ and plot $E_m$ and $E_s$ for each prior.
In particular, as stated in lemma \ref{thm:all_prior}, with $n^l = n^{l-1} = 64$, we have $\mathbb{E}_s > \mathbb{E}_m$ for all priors.
We also observe that splitting performs better for larger input dimensions while merge performs better for smaller input dimensions. 
This shows the complementary between these two steps.
The uniform prior appears to be the least favorable while the exponential is the most favorable. 
However a limit to the theory is the i.i.d. hypothesis which is probably not satisfied in practice. For this reason, we evaluate the expected pruning in order to evaluate the modelization as a birthday problem.
%
%
\subsection{Birthday Problem and Empirical Distributions}\label{sec:bday_practice}
\begin{figure}[!t]
    \centering
    \includegraphics[width = 0.973\linewidth]{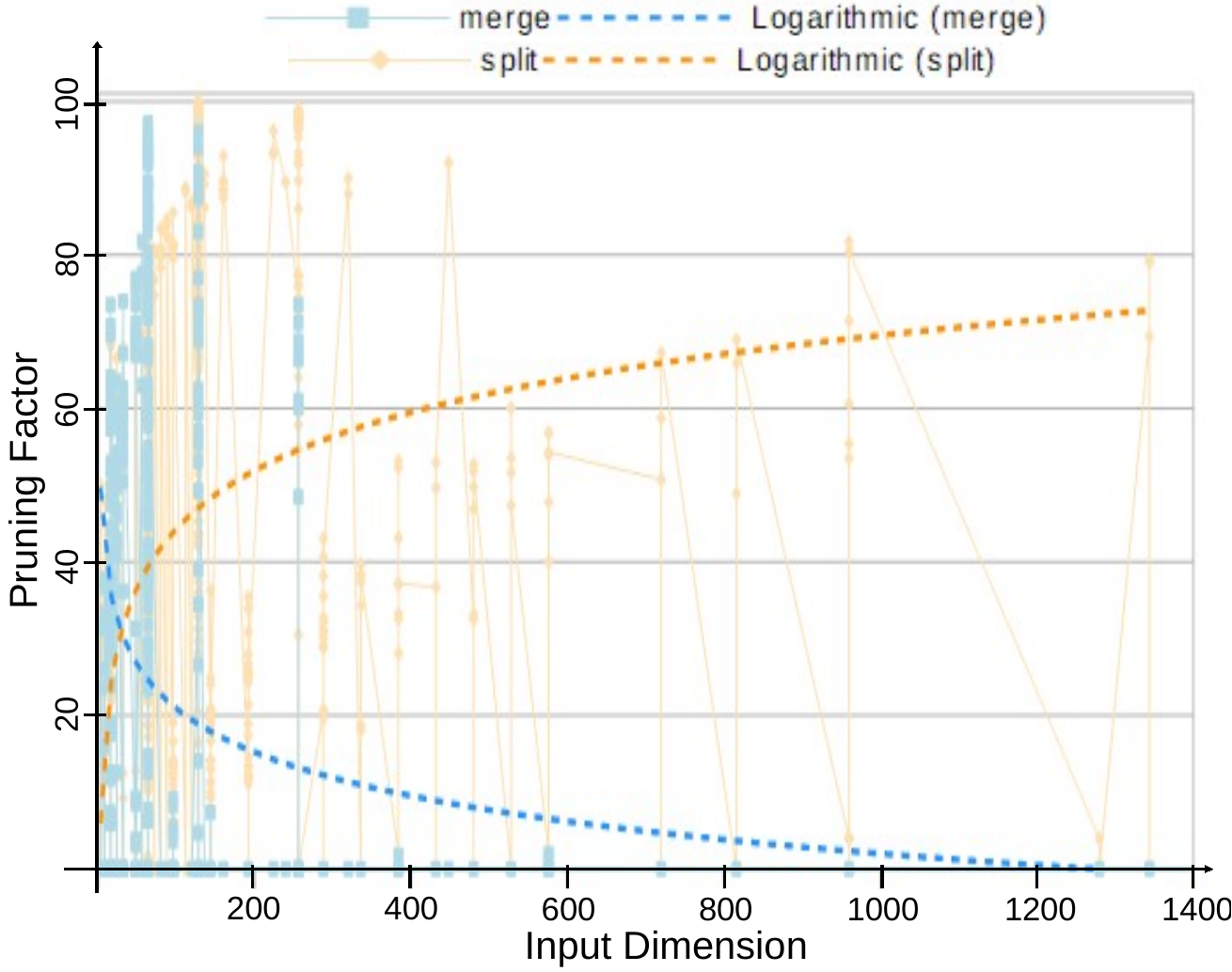}
    \caption{Empirical distribution of the pruning ratios from merging (blue) and splitting (orange) on several layers of different networks (e.g. ResNets and MobileNets) on different tasks, namely ImageNet and Cifar10. This corresponds to the \textit{in situ} results from Fig \ref{fig:Es_Em_all_priors}. Although very noisy, we observe similar trend lines as the theoretical expected pruning ratios.
    }
    \label{fig:bday_empirical}
\end{figure}
We study several architectures with a wide range of input shapes per layers. 
In Fig \ref{fig:bday_empirical}, we compute the empirical pruning ratio per input dimension similarly to Fig \ref{fig:Es_Em_all_priors} for a number of layers in various ResNets and MobileNets architectures on ImageNet and Cifar10.
The trend curve of pruning ratios from merge appears to match the theoretical results with a inverse proportionality relation with the input dimension.
In the case of split, on the contrary, the pruning ratio increases as a function of the input dimension which confirms the complementary with merging.
However, in detail we observe two trends with some highly pruned small layers which correspond to squeeze layers of the MobileNets and large layers that are not pruned which correspond to depthwise layers and prediction heads (not processed fully connected layers).
All in all, these results confirm the validity of the theoretical study of the expected pruning from \oursvir.
\newline
We empirically demonstrated that the lower the remaining number of distinct weights after hashing the higher the expected pruning from merging and splitting.
We also justified the complementarity of these steps for a given layer size by formulating it as a birthday problem.
We now experimentally validate the proposed \ours method quantitatively and qualitatively.

%
%
%
%

\section{Experimental Validation of \ours}
\begin{table}[!t]
\renewcommand{\arraystretch}{1.15}
\caption{Each pruning step of \ours are applied in a sequence. The merge step with $\alpha = 0$ is always applied. In this ablation study we test the impact of each step.  We report for different models and tasks the accumulated pruning ratio of each one of them. The considered metric is the percentage of removed parameters. We note in bold the best performance and underline the second best.}
\label{tab:ablation}
\centering
\setlength\tabcolsep{4pt}
\begin{tabular}{|c|c|c|c|c|c|c|}
\hline
Hashing                     & \xmark & \xmark & \cmark & \cmark & \cmark & \cmark \\
merge ($\alpha$ value)      & 0      & $\alpha^*$ & 0      & $\alpha^*$ & 0      & $\alpha^*$ \\
split                       & \cmark & \cmark & \xmark & \xmark & \cmark & \cmark \\
\hline
\hline
\multicolumn{7}{|c|}{\cellcolor{Gray}Cifar10}\\
\hline
ResNet 20            & 0.00 & 18.58 & 25.18 & 41.03 & \underline{65.26} & \textbf{67.48} \\
ResNet 56            & 0.00 & 61.19 & 58.45 & 77.68 & \underline{85.89} & \textbf{88.81} \\
ResNet 110           & 0.00 & 75.29 & 62.41 & 84.43 & \underline{88.31} & \textbf{91.82} \\
ResNet 164           & 0.00 & 78.61 & 62.73 & 88.87 & \underline{90.87} & \textbf{94.49} \\ 
\hline
Wide ResNet 28-10   & 0.00 & 47.25 & 25.59 & 58.79 & \underline{77.49} & \textbf{80.13} \\
Wide ResNet 40-4    & 0.00 & 49.67 & 43.37 & 61.80 & \underline{65.97} & \textbf{68.59} \\
\hline
MobileNet V2 (0.35) & 0.00 & 4.99 & 0.00  & 6.47  & \underline{53.65} & \textbf{55.48} \\ 
MobileNet V2 (0.5)  & 0.00 & 2.97 & 0.00  & 3.87  & \underline{57.42} & \textbf{59.70} \\ 
MobileNet V2 (0.75) & 0.00 & 2.21 & 0.00  & 2.87  & \underline{63.14} & \textbf{65.65} \\ 
MobileNet V2 (1)    & 0.00 & 4.21 & 0.00  & 5.46  & \underline{69.17} & \textbf{71.52} \\ 
MobileNet V2 (1.4)  & 0.00 & 2.43 & 0.00  & 3.16  & \underline{77.92} & \textbf{80.57} \\ 
\hline
EfficientNetB0      & 0.00 & 2.02 & 1.44  & 2.62  & \underline{61.79} & \textbf{64.62} \\
EfficientNetB1      & 0.00 & 2.31 & 1.35  & 2.99  & \underline{66.12} & \textbf{68.75} \\
EfficientNetB2      & 0.00 & 4.53 & 1.25  & 5.86  & \underline{68.53} & \textbf{71.26} \\
EfficientNetB3      & 0.00 & 2.57 & 1.17  & 3.33  & \underline{70.28} & \textbf{72.67} \\
EfficientNetB4      & 0.00 & 3.81 & 1.04  & 4.93  & \underline{73.42} & \textbf{76.34} \\
EfficientNetB5      & 0.00 & 4.08 & 0.89  & 5.29  & \underline{77.66} & \textbf{80.30} \\
EfficientNetB6      & 0.00 & 3.88 & 0.81  & 5.03  & \underline{78.83} & \textbf{81.97} \\
EfficientNetB7      & 0.00 & 2.13 & 0.72  & 2.76  & \underline{80.23} & \textbf{83.42} \\
\hline
\hline
\multicolumn{7}{|c|}{\cellcolor{Gray}ImageNet}\\
\hline
ResNet 50           & 0.00 & 0.09 & 0.09 & 0.29 & \underline{43.95} & \textbf{44.25} \\
ResNet 101          & 0.00 & 0.75 & 0.58 & 0.75 & \underline{44.12} & \textbf{44.51} \\
ResNet 152          & 0.00 & 0.58 & 0.58 & 0.58 & \underline{43.64} & \textbf{43.68} \\
\hline
MobileNet V2 (0.35) & 0.00 & 2.75 & 0.00 & 2.75 & \underline{14.13} & \textbf{14.88} \\
MobileNet V2 (0.5)  & 0.00 & 1.40 & 0.00 & 1.40 & \underline{22.35} & \textbf{22.94} \\
MobileNet V2 (0.75) & 0.00 & 1.50 & 0.01 & 1.51 & \underline{33.27} & \textbf{35.75} \\
MobileNet V2 (1)    & 0.00 & 1.95 & 0.01 & 1.95 & \underline{46.00} & \textbf{46.97} \\
MobileNet V2 (1.4)  & 0.00 & 2.16 & 0.02 & 2.18 & \underline{83.73} & \textbf{85.42} \\
\hline
EfficientNetB0      & 0.00 & 1.68 & 1.23 & 1.91 & \underline{53.52} & \textbf{54.43} \\
EfficientNetB1      & 0.00 & 1.85 & 1.18 & 2.10 & \underline{62.68} & \textbf{63.17} \\
EfficientNetB2      & 0.00 & 1.39 & 1.10 & 1.58 & \underline{66.37} & \textbf{68.35} \\
EfficientNetB3      & 0.00 & 2.46 & 1.06 & 2.80 & \underline{69.79} & \textbf{71.41} \\
EfficientNetB4      & 0.00 & 2.43 & 0.98 & 2.76 & \underline{72.26} & \textbf{74.58} \\
EfficientNetB5      & 0.00 & 3.60 & 1.06 & 4.09 & \underline{80.64} & \textbf{81.41} \\
EfficientNetB6      & 0.00 & 1.57 & 1.70 & 1.79 & \underline{85.28} & \textbf{87.88} \\
EfficientNetB7      & 0.00 & 2.23 & 2.54 & 2.54 & \underline{87.50} & \textbf{89.21} \\
\hline
\end{tabular}
\end{table}
%
The proposed \ours method is composed by three steps: an adaptive data-free hashing, an output-wise merging step (with hyperparameter $\alpha$) and an input-wise splitting step (presented as algorithm in Appendix \ref{sec:appendix_algo}).

In order to validate our approach, we first perform an extensive ablation study to precisely assess the utility of each of these steps on popular deep convolutional architectures which are to this day the most popular computer vision model family. We also propose a qualitative analysis on these models to more precisely understand how the networks are pruned. We then show how our approach can be applied to different architectures that use fully-connected layers, such as transformers. Last but not least, we compare our method with state-of-the art methods, showing that it significantly outperforms other data-free methods and often rivals data-driven ones.

%
%
%
\subsection{Quantitative Analysis}\label{convnetablation}
First we perform ablation study of each individual component in \oursvir. Namely, for each experiment, we measure the pruning factor defined as the proportion of removed parameters from the original model (e.g. a pruning factor of $100\%$ indicates that the entire network was pruned). The results are presented in TABLE \ref{tab:ablation}, where each column corresponds to a different combination of \ours elementary blocks: for instance, the second column indicate no hashing, merge with $\alpha = 0$ and split while the fifth column indicates hashing, $\alpha = \alpha^*$ and no split.
As shown in the second column, merge with $\alpha = 0$ and no hashing leads to no pruning at all for every network: this shows that an approximation is required to introduce redundancies, whether it is weight value hashing or merging relaxation ($\alpha = \alpha ^*$). However, the third and fourth columns shows that just using either of these approximations results in very low pruning rates. In all cases, as shown in column five, using both hashing and merging relaxation ($\alpha = \alpha ^*$) leads to higher pruning ratios, most notably on ResNet (e.g. reaching $88.9$\% parameters removed on ResNet 164) and Wide ResNet networks on Cifar10. However, the pruning ratios are still very low on more compact networks (e.g. MobileNet and EfficientNet backbones) or ImageNet-trained DNNs.


Adding the splitting step (sixth column) allows to achieve superior pruning rates on any network on both Cifar10 and ImageNet. Furthermore, using the relaxed merging step ($\alpha = \alpha^*$) along with the splitting step improves the results by $2$ to $4$ points on every network.
The proposed method achieves remarkable results across all networks: on the one hand, on already very small and compact networks designed for efficiency (e.g. MobileNet V2 network family or EfficientNet B0), the combination of hashing, merging and splitting allows to remove $53.65-77.92$\% on Cifar10 and $14.13-83.73$\% on ImageNet. Note that for a MobileNet network on ImageNet, nearly $80$\% of the network's parameters are contained in the last fully connected layer which is not considered as we only prune convolutional layers in this experiments, hence the lower baselines, notably with very low width multiplier (0.35-0.75).
On the other hand, on larger networks, e.g. on ResNets, Wide ResNets and the larger EfficientNets, \ours composed of hashing, merging and splitting typically removes $64.62-94.49$\% on Cifar10 and $43.68-89.21$\% of the parameters on ImageNet. This shows that \ours is very versatile and achieve superior pruning performance on many state-of-the-art networks, most notably without any drop in accuracy and without extensive hyperparameter setting, thanks to the complementarity between merging and splitting steps, the former working better on layers with large output dimension and small input dimensions, while the latter achieves better performance with larger input dimensions, as discussed in Section \ref{sec:generalized_bday}). 
We provide a qualitative analysis of the per-layer behaviour in Appendix \ref{sec:appendix_qualitative_analysis}.

%
%
%
%
\subsection{split and Fully Connected Layers}\label{sec:appendix_fully_connected_layers}
\begin{table}[!t]
\renewcommand{\arraystretch}{1.15}
\caption{\ours evaluation on transformer architectures on ImageNet. These architecture embrace large fully connected layers. For each network, we report the total number of parameters, the hashing reduction rate (or hashing ratio), the accuracy drop, and the pruning ratios obtained by applying merging with $\alpha = 0$ as well as splitting. For all models, \ours achieves very high pruning ratios with no accuracy drop.}
\label{tab:fully_connected}
\centering
\begin{tabular}{|c|c|c|c|c|}
\hline
 model & $|W|$ & hashing ratio & acc drop & pruning ratio \\
\hline
\hline
DeiT T          & 5.7M  & 96.016\% & 0.000 & 70.32\% \\
DeiT S          & 22M   & 97.994\% & 0.100 & 84.93\% \\
DeiT            & 87M   & 98.527\% & 0.000 & 93.20\% \\
\hline
CaiT XS24       & 26.7M & 98.028\% & 0.000 & 81.85\% \\
CaiT S24        & 47M   & 98.992\% & 0.030 & 86.27\% \\
CaiT M36        & 271M  & 99.443\% & 0.100 & 91.80\% \\
\hline
LeViT 128S      & 7.9M  & 97.604\% & 0.000 & 76.91\% \\
LeViT 128       & 9.4M  & 97.774\% & 0.000 & 77.46\% \\
LeViT 192       & 11M   & 96.925\% & 0.000 & 80.49\% \\
LeViT 256       & 19M   & 97.133\% & 0.000 & 85.93\% \\
LeViT 384       & 39M   & 97.214\% & 0.000 & 90.55\% \\
\hline
\end{tabular}
\end{table}
In what precedes, we tested \ours on a wide range of convolutional architectures which are one ubiquitous model family for computer vision applications. Transformers are another example of such promising architectures, achieving state-of-the-art performance on ImageNet. 
For this reason, we validate \ours on today's most successful image transformer architectures, namely DeiT \cite{touvron2020training}, CaiT \cite{touvron2021going} and LeViT \cite{graham2021levit}.
For DeiT we considered the base model as well as the small (S) and Tiny (T) versions with input shape 224. For CaiT we considered Cait Extra Small (XS24), Small (S24) and Medium (M36) with depth 24, 24 and 36, respectively. Lastly, for Levit \cite{graham2021levit} we considered the five scale models, \textit{i.e.} 128S, 128, 192, 256 and 384.
We report our results in TABLE \ref{tab:fully_connected}. For all transformer networks, we measure the effectiveness of the hashing (hashing ratio) as well as the induced accuracy drop. Furthermore, we report the pruning ratio on each models with hashing, merging with $\alpha = 0$ and splitting. 
\newline\textbf{Hashing:} Similarly to convolutional networks, the image transformers are highly compressed \textit{via} the proposed adaptive hashing without witnessing any significant accuracy drop. As a matter of facts, the percentage of removed parameters is higher than that of convolutional networks with a comparable number of parameters,
e.g. in EfficientNet B0 hashing removes $72.3$\% distinct weight values while it removes $96$\% on DeiT T. This is also the case on larger networks with $97.8$\% on EfficentNet B7 and $98.5$\% on DeiT.
This is a consequence of the fact that more layers were ignored in convolutional networks (e.g. depthwise convolutional layers) and also the fact that Image Transformers usually have less layers but larger ones.
\newline\textbf{Pruning:} Generally speaking, \ours works very well on transformers, \textit{i.e.} from $70\%$ on the smaller models (e.g. DeiT T, LeViT 128S) to more than $90\%$ on larger ones (DeiT, CaiT M36, LeViT 384). Because transformers usually have fewer layers with larger input dimensions (see Section \ref{sec:generalized_bday}), the splitting step allows to remove large numbers of operations as compared with convolutional architectures: for instance on a models with $\approx 5$M parameters we jump from 55\% pruning ratio to 77\%. Also note that, echoing the results reported in Section \ref{convnetablation}, here again hashing appears as an essential step for introducing redundancies in DNNs as we obtained 0\% pruning ratios for all networks without it. 

\subsection{Comparison with State-Of-The-Art}\label{sec:experiments_SOTA}
%
%
%
\begin{figure*}[!th]
    \centering
    \includegraphics[width = \linewidth]{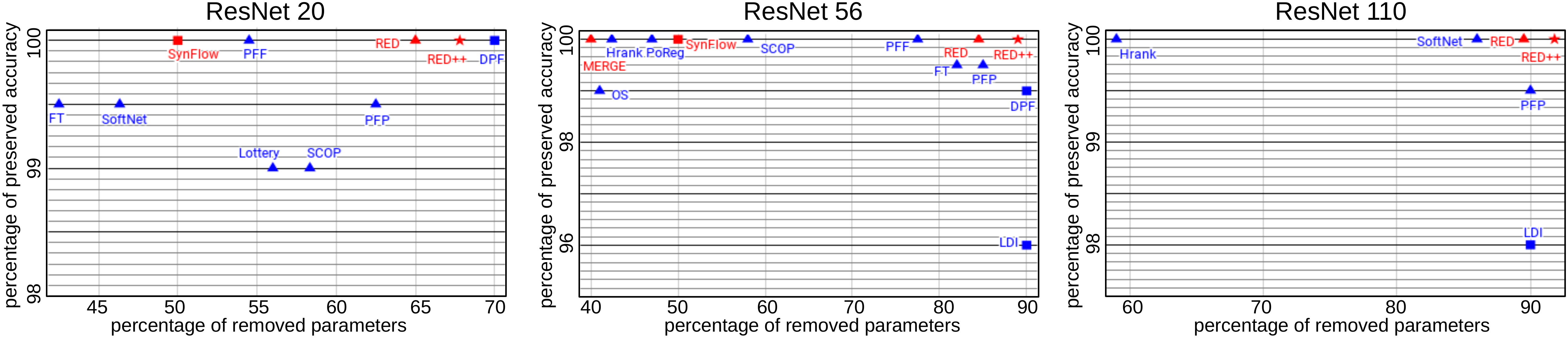}
    \caption{Comparison of \ours and SOTA methods for ResNet 20, 56 and 110 on Cifar10 in terms of percentage of removed parameters (horizontal axis) and accuracy preservation (vertical axis). Data-free methods are in red while data-driven ones are in blue. Structured methods are plotted as squares and unstructured ones as triangles.}
    \label{fig:SOTA_cifar}
\end{figure*}
\begin{figure*}[!th]
    \centering
    \includegraphics[width = \linewidth]{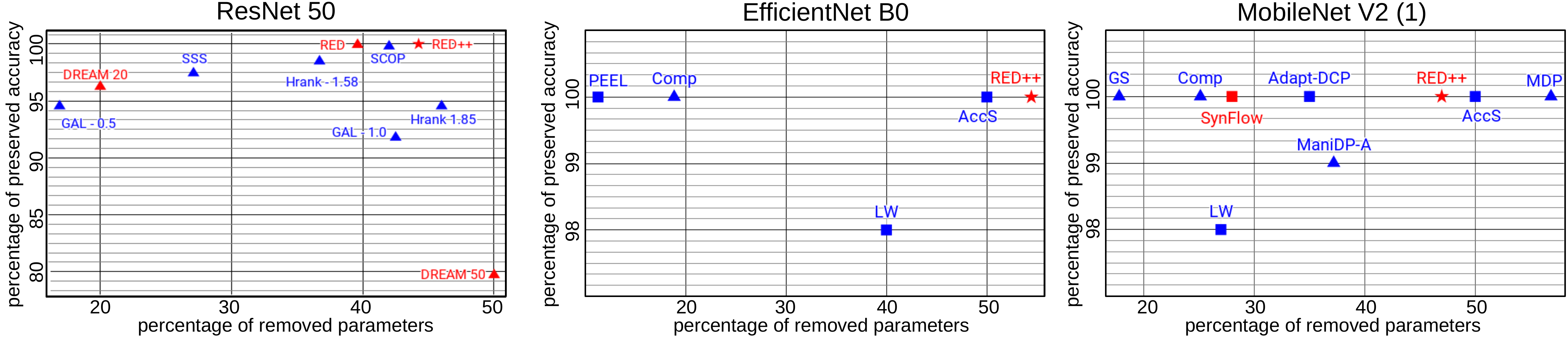}
    \caption{Comparison of \ours and SOTA methods for EfficientNet B0, ResNet 50 and MobileNet V2 on ImageNet in terms of percentage of removed parameters (horizontal axis) and accuracy preservation (vertical axis). Data-free methods are in red while data-driven ones are in blue. Structured methods are plotted as squares and unstructured ones as triangles.}
    \label{fig:SOTA_mob&eff}
\end{figure*}
In this section, we compare \ours with recent state-of-the-art DNN pruning methods. These methods can be referenced based on two criteria which strongly influence the pruning performance. First, the usage of data, characterized by the color (data-free in red and data-driven in blue). Data-driven methods, e.g. \cite{hou2021network,wang2020accelerate,horton2020layer,marban2020learning,mishra2021accelerating,bhalgat2020structured,ye2020good,guo2020multi,tang2021manifold,liu2021discrimination,zhuang2020neuron} usually vastly outperform data-free \cite{tanaka2020pruning, kim2020neuron, red2021} ones as they benefit from re-training the network to make up for the accuracy loss that may be caused by the pruning. The second classification method lies in the type of sparsity imposed to the resulting network. As studied in \cite{renda2020comparing}, unstructured pruning can easily achieve far better pruning ratios than structured pruning, which is usually more constrained. For this reason, we distinguish structured approaches, e.g. \cite{wang2020accelerate,bhalgat2020structured,ye2020good,guo2020multi,kim2020neuron, tang2021manifold, red2021} from unstructured ones, e.g. \cite{liu2021discrimination, hou2021network, horton2020layer, marban2020learning, mishra2021accelerating, tanaka2020pruning} using shapes (structured methods are highlighted with triangles and unstructured methods with a square).
\ours is data-free and structured, and, for the sake of clarity, is represented with red star. The results are shown on Figures \ref{fig:SOTA_cifar} and \ref{fig:SOTA_mob&eff} for several network architectures on which state-of-the-art methods also reported results.
\newline\textbf{Comparison on Cifar10: }
Fig. \ref{fig:SOTA_cifar} draws a comparison in terms of pruning performance for several state-of-the-art algorithms, measured as the trade-of between the proportion of removed parameters and proportion of the preserved accuracy from most popular models on Cifar10, i.e. ResNet 20, 56 and 110.
Generally speaking, \ours vastly outperforms all other data-free pruning techniques with our previous work RED being its closest contender.
Furthermore, \ours allows to greatly narrow the performance gap between data-free and data-driven pruning methods: on ResNet 20 and ResNet 56, it enables pruning rates similar to state-of-the-art DPF with no accuracy loss. On ResNet 110, \ours outperforms all other methods by $3$ points  on a benchmark with already high pruning rates ($>90$\%).
Excluding our previous work RED, \ours outperforms other data-free methods by $42$\% on ResNet 56 which is the most common benchmark, according to \cite{blalock2020state}.
\newline\textbf{Comparison on ImageNet: }
Fig. \ref{fig:SOTA_mob&eff} draws a comparison between \ours and other state-of-the-art pruning approaches on ImageNet. First, on ResNet 50 (Fig. \ref{fig:SOTA_mob&eff}-left plot) show \ours ability to preserve accuracy while achieving higher pruning ratios than any other structured methods. It outperforms SCOP \cite{tang2020scop} by $4$\% pruning rate and achieves $5.5$\% higher top1 accuracy than Hrank \cite{lin2020hrank} for a similar pruning rate, even though both methods are data-driven. When compared to data-free approaches, \ours outperforms RED \REDref by $4$\% pruning rate with no accuracy drop and significantly outperforms DREAM \cite{yin2020dreaming}, either by allowing superior pruning rates ($>20$\% pruning rate as compared with DREAM-20) or by retaining the original model accuracy ($>20$\% accuracy drop for DREAM-50) with a slightly lower pruning rate.
We also provide comparison on less common benchmarks such as MobileNet V2 and EfficientNets (Fig. \ref{fig:SOTA_mob&eff}). These networks are designed for efficiency and are more challenging.
Nonetheless, \ours outperforms all previous work by far, with an improvement of $4.4$\% upon AccS \cite{mishra2021accelerating} on EfficientNet B0 while preserving the accuracy.
On MobileNet V2 (with width multiplier 1), \ours doesn't prune the last layer which represents $35$\% of the network and still outperforms data-free methods such as SynFlow \cite{tanaka2020pruning} by $67.75$\%.
We also approach the  performance of recent data-driven methods such as MDP \cite{guo2020multi}. Note that if we prune the last layer, \ours achieves $66.12$\% pruning ratio with $5.43$\% accuracy drop.
\newline
Thus, we showed that \ours achieves high pruning rates on multiple datasets and networks, ranging from traditional convolutional architectures (e.g. ResNet family) to more compact and efficient designs (e.g. MobileNets and EfficientNets), as well as recently proposed image transformers. It outperforms other data-free and structured pruning methods by a significant margin, and narrows the gap with state-of-the-art data-driven and unstructured methods, despite being more constrained.


%
%
%
%
\section{Discussion and conclusion}\label{sec:conclusion}
In this paper, we proposed a novel data-free DNN pruning method, called \oursvir. This method is composed of 3 steps: First, a layer-wise adaptive scalar weight hashing step which dramatically reduces the number of distinct weight values. Second, an output-wise redundant neuron merging. Third, an input-wise splitting strategy.

For the hashing step, we propose theoretical guarantees for DNN accuracy preservation, by bounding the induced error w.r.t. the original model. Furthermore, we propose a data-free criterion to assess the hashing behavior using only the DNN weights and batch normalization layer statistics. We experimentally validate this adaptive hashing step on multiple networks and datasets, showing its benefits as a standalone memory footprint solution as well as to highlight redundancies in DNNs, e.g. as compared as baseline solutions such as uniform weight quantization.

These induced redundancies can then be exploited in the frame of a novel pruning scheme, that includes an output-wise merging step as well as an input-wise splitting step. We formalized the pruning problem as a generalized birthday problem, which provides theoretical guarantees on the effectiveness of the proposed method. In particular, we derived expected values for the merging and splitting step pruning rates, showing the complementarity between these two steps. We experimentally validate this theoretical analysis, showing that merging and splitting are indeed complementary and that their association allows very high pruning rates.

Last but not least, we thoroughly evaluated our method on several datasets and network architectures, showing that \ours achieves very high pruning rates on traditional computer vision network families (e.g. ResNets), more compact architectures (e.g. MobileNets and EfficientNets), as well as recently proposed image transformers. As such, \ours significantly outperforms state-of-the-art data-free and structured pruning methods, and substantially narrows the gap with less constrained data-driven or unstructured methods on every benchmark. Furthermore, as detailed in Appendix \ref{sec:appendix_multi_dupli}, these high pruning rates generally translate well into FLOPs reduction. All in all, we believe that \ours will pave the way for innovation in DNN acceleration, and even the design of specific hardware and software solutions to facilitate not only matrix multiplication but also memory allocation speed, which is paramount to the proposed splitting method performance. As such, we provide an analysis of the splitting step performance in terms of direct runtime acceleration using existing hardware and software solution in Appendix \ref{sec:appendix_flops}. Our conclusion is that while input-wise splitting allows to greatly reduce the computational burden with standard hardware and software, more acceleration can be achieved with more dedicated solutions.

Future work involves using \ours in conjunction with other DNN compression techniques, by using the proposed hashing step along with existing data-free quantization techniques. Furthermore, other pruning methods, such as magnitude-based sparse pruning, could be used on top of \ours for higher compression rates. Ultimately, the data-free hypothesis can be relaxed, and networks with split layers could be re-trained for even more computationally efficient DNN architectures.

\begin{IEEEbiography}[{\includegraphics[width=1in,height=1.25in,clip,keepaspectratio]{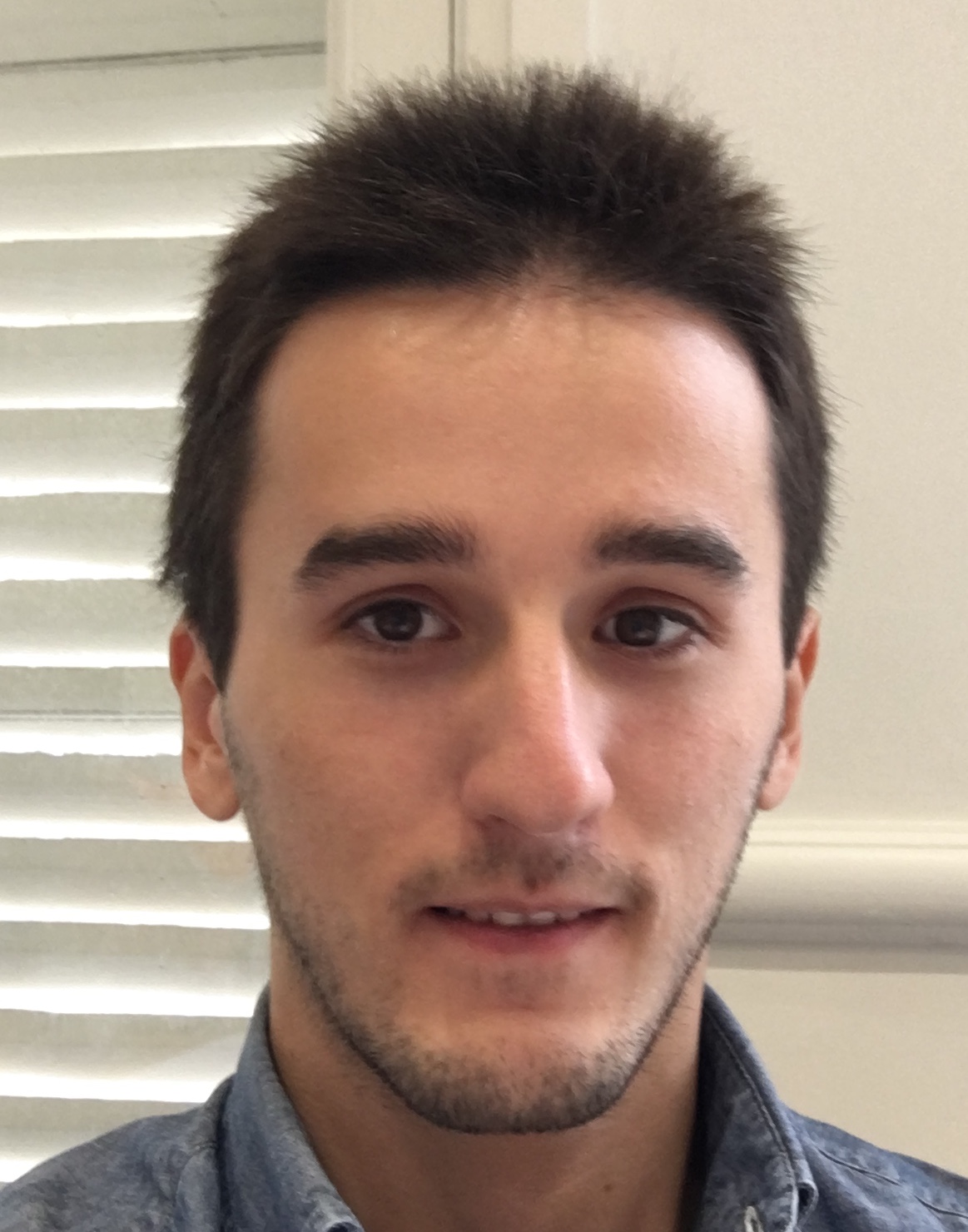}}]{Edouard YVINEC}
received his master's degree from Ecole Normale Superieure Paris-Saclay in 2020 and is currently a P.h.D. student at ISIR in Sorbonne Université. His research interest include but are not limited to DNN solutions for computer vision tasks, compression and acceleration of such models.
\end{IEEEbiography}
\begin{IEEEbiography}[{\includegraphics[width=1in,height=1.25in,clip,keepaspectratio]{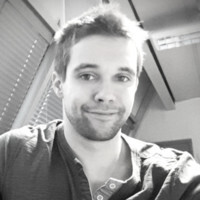}}]{Arnaud DAPOGNY} is a computer vision researcher at Datakalab in Paris. He obtained the Engineering degree from the Sup\'elec engineering School in 2011 and the Masters degree from Sorbonne University, Paris, in 2013 with high honors. He also obtained his PhD at Institute for Intelligent Systems and Robotics (ISIR) in 2016 and worked as a post-doctoral fellow at LIP6. His works concern deep learning for computer vision and its application to automatic facial behavior as well as gesture analysis.
\end{IEEEbiography}
\begin{IEEEbiography}[{\includegraphics[width=1in,height=1.25in,clip,keepaspectratio]{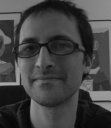}}]{Matthieu CORD}
 is full professor at Sorbonne University. He is also part-time principal scientist at Valeo.ai. His research expertise includes computer vision, machine learning and artificial intelligence. He is the author of more 150 publications on image classification, segmentation, deep learning, and multimodal vision and language understanding. He is an honorary member of the Institut Universitaire de France and served from 2015 to 2018 as an AI expert at CNRS and ANR (National Research Agency).
\end{IEEEbiography}
\begin{IEEEbiography}[{\includegraphics[width=1in,height=1.25in,clip,keepaspectratio]{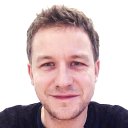}}]{Kevin BAILLY CORD}
is associate professor with the Institute of Intelligent Systems and Robotics (ISIR) at Sorbonne University and Head of Research of Datakalab. He received the PhD degree in computer science from the Pierre et Marie Curie University in 2010 and was a postdoctoral researcher at Telecom Paris from 2010 to 2011. His research interests are in machine learning and computer vision applied to face processing and behavior analysis.
\end{IEEEbiography}

{
\bibliographystyle{ieee_fullname}
\bibliography{egbib}
}
\newpage $ $
\newpage 

%
%
%
%
\begin{appendices}
%
%
%
%
\section{Details on the Upper Bound $U$}\label{sec:appendix_U_details}
We detail the computation of equations \ref{eq:hash_per_layer} and \ref{eq:hashing_2}.
A convolution is defined by the following operation
\begin{equation}
\text{Output}_{i,j,l} = \sum_{\delta i, \delta j, k} \text{Input}_{i+\delta i, j + \delta j, k} W_{i+\delta i, j + \delta j, k, l}
\end{equation}
This gives us $n^{l-1}w^lh^l$ multiplications per output. 
For each of these operations we have the upper bound $u_l$. 
Therefore, the errors are sampled in $[-u_l; u_l]$.
The Central limit theorem gives us that the average error converges to a standard Gaussian distribution $\mathcal{N}(0,1)$ and we get that
\begin{equation}
\mathbbm{E}[\|\tilde f^l - f^l\|] \leq \frac{u_l}{\sqrt{n^{l-1}w^lh^l}}
\end{equation}
assuming no activation function on layer $l$.
Now let's assume we have $L =2$, \text{i.e.} the DNN $f : x \mapsto f^2(f^1(x))$ follows
\begin{equation}
\mathbbm{E}_X[\| \tilde f - f \|] = \mathbbm{E}_X[\| \tilde f^2(\tilde f^1(X)) - f^2(f^1(X)) \|]
\end{equation}
where $X$ is the random variable defined by the inputs.
Because $f$ is piece-wise affine, we can assert that 
\begin{equation}
\begin{aligned}
\mathbbm{E}_X[\| \tilde f - f \|] \simeq& \mathbbm{E}_X[\| \tilde f^2(f^1(X)) + \tilde f^2(\|\tilde f^1(X) - f^1(x)\|)\\ & - f^2(f^1(X)) \|
]
\end{aligned}
\end{equation}
Now using the triangular inequality, we get
\begin{equation}
\mathbbm{E}_X[\| \tilde f - f \|] \simeq \mathbbm{E}_{f^1(X)}[\|\tilde f^2 - f^2\| + \mathbbm{E}_{\|\tilde f^1(X) - f^1(x)\|}[\tilde f^2]
\end{equation}
We deduce
\begin{equation}
\mathbbm{E}_X[\| \tilde f - f \|] \leq \mu^1\mathbbm{E}[\|\tilde f^2 - f^2\|] + \mu^2\mathbbm{E}[\|\tilde f^1 - f^1\|]
\end{equation}
and it follows equation \ref{eq:hashing_2}.

%
%
%
%
\section{Datasets and Base Models}\label{sec:appendix_datasets}
We evaluate our models on the two \textit{de facto} standard datasets for pruning, \textit{\textit{i.e.}} CIFAR-10 \cite{krizhevsky2009learning} (50k train/10k test) and ImageNet \cite{ImageNet_cvpr09} ($\approx 1.2$M images train/50k test).
In order to compare our protocol to the SOTA we considered the most frequently used networks among which ResNet \cite{he2016deep} (ResNet 20 and 56 with respective number of parameters 270k, 852k and accuracies $92.48$, $93.46$ on CIFAR-10 and ResNet 50 with 25M parameters and $76.17$ top-1 accuracy on ImageNet) as well as Wide-ResNet \cite{zagoruyko2016wide}.
Wide ResNet architectures are defined by their number of layers as well as their wideness multiplier: we evaluate on Wide ResNet 28-10 with 36.5M parameters and $95.8$ accuracy and Wide ResNet 40-4 with 8.9M parameters and $95.0$ accuracy.\newline
We also tested our method on networks presenting different types of convolutional layers layer such as MobileNet V2 \cite{sandler2018mobilenetv2} and EfficientNet \cite{tan2019efficientnet} (depthwise-separable convolutions). MobileNets are tested with different width-multiplier from MobileNet V2 (0.35) having $1.66$M parameters and $60.3$ top1 accuracy on ImageNet to MobileNet V2 (1.4) having $6.06$M parameters and $75.0$ top1 accuracy. EfficientNet is tested for different scales from B0 with $5.3$M parameters and $77.2$ top1 accuracy to B7 with $66.7$M (B7) parameters and $84.4$ accuracy.\newline 
Furthermore we tested the method on networks using large fully connected layers such as Image Transformers \cite{touvron2020training, touvron2021going,graham2021levit}.

%
%
%
%
\section{Implementation Details}\label{sec:appendix_implem_details}
The hashing step was implemented following the guidelines provided in \REDref, using Scikit-learn python library, with bandwidth $\Delta_l$ set as the median of the differences between consecutive weight values per layer $l$ in the adaptive hashing algorithm. 
The merging step is implemented using Numpy. Contrary to the original method, we set the hyperparameter $\alpha$ to its optimal value $\alpha^*$ (\textit{i.e.} the highest $\alpha$ which preserves the accuracy of the model) unless stated otherwise. We apply a per block strategy for setting the layer-wise $(\alpha^l)$ detailed in Appendix \ref{sec:appendix_alpha}. 
The splitting step was also implemented in Numpy python library.\newline
We ran our experiments on a Intel(R) Core(TM) i7-7820X CPU. Similarly to \REDref, the hashing step is the bottleneck and its processing time depends on the model's size: ranging from a few seconds for smaller models suited for CIFAR-10 and up to a day for a wider networks designed for ImageNet. 
For a given layer $l$, the complexity of the proposed hashing is $\mathcal{O}(|W^l||S|)$ where $|W^l|$ denotes the number of weights and $|S|$ the sampling size (used for the evaluation of the kde and corresponds to $\omega$ in equation 1). The hashing can be accelerated (e.g. up to 50 times faster on a ResNet 50) by processing the layers in parallel. Furthermore, for very large layers (e.g. layers with over $|W^l| = 10^6$ parameters) we can simply take a fraction (e.g. $5.10^4$ values in $W^l$) of the weights randomly to get identical density estimators much faster. Following the example, we would go $\frac{10^6}{5.10^4} = 20$ times faster, due to the linear complexity. As such, we were able to process ResNet 101 and 152 with $99.9\%$ compression and accuracy drop, in about 2 and 4.5 hours respectively.
The pruning steps (merging and split) only require a few minutes in the worst case scenario, never exceeding half an hour.

\section{Merging Hyperparameter Selection Strategy}\label{sec:appendix_alpha}
The merging parameter $\alpha$ defines the average proportion of neurons to merge per layer. The merging step is performed layer per layer and we study four candidate strategies to derive $\alpha^l$ values from $\alpha$. The selected strategy, called \textit{per block strategy}, is defined as follows
\begin{equation}
\begin{cases}
    \alpha^l = \max\{2\alpha -1,0\} & \text{if } l \in \llbracket 0 ; L/3 \llbracket \\
    \alpha^l = \alpha & \text{if } l \in \llbracket L/3 ; 2L/3 \rrbracket \\
    \alpha^l = \min\{2\alpha,1\} & \text{if } l \in \rrbracket 2L/3 ; L \rrbracket \\
\end{cases}
\end{equation}
The other strategies are
\begin{itemize}
    \item \textit{constant} strategy: $\forall l \in \llbracket 1 ; L \rrbracket$, $\alpha^l$
    \item \textit{linear ascending} strategy: $\alpha^l$ $\forall l \in \llbracket 1 ; L \rrbracket$, $\alpha^l = \alpha l / L$
    \item \textit{linear descending} strategy: $\forall l \in \llbracket 1 ; L \rrbracket$, $\alpha^l = \alpha (L-l) / L$
\end{itemize}
The conclusion remains unchanged and the per-block strategy is used in all our other benchmarks.

In the main paper we considered two possible values for $\alpha$.
The first corresponds to the prior-less situation where we want to maximize the accuracy of the model, \textit{i.e.} we only merge strictly identical hashed neurons, \textit{i.e.} $\alpha = 0$.
The second serves only the purpose of comparison as its not data-free and corresponds to maximum pruning without accuracy drop from hashing, that we noted $\alpha = \alpha^*$.
Note that it only requires validation and no training to be leveraged.
As stated in the main paper the parameter $\alpha$ can be set at different values for different pruning/accuracy trade-offs.
Note that as a general rule of thumb, $\alpha$ can be set to relatively high values (e.g. $20\%$ to 50\%) on easier tasks (e.g. Cifar10) but should remain low (lower than $5\%$) on harder tasks (e.g. ImageNet).

\begin{table}[!t]
\centering
\renewcommand{\arraystretch}{1.15}
\setlength\tabcolsep{2.75pt}
\caption{Comparison between different strategies for $\alpha^l$ in terms of pruning factor for several networks on Cifar10. We use the same hashing for all methods and report the pruning factor for merging $\alpha = \alpha^*$.}
\begin{tabular}{l|c|c|c}
\hline
Strategy & ResNet 56 & MobileNet V2 (1) & Wide ResNet 40-4 \\
\hline\hline
linear descending & 71.67 & 1.19 & 57.22 \\
\hline
constant  & 72.40 & 0.87 & 58.78 \\
\hline
linear ascending & 73.93 & 3.01 & 60.48 \\
\hline
block  & \textbf{77.68} & \textbf{5.46} & \textbf{61.80} \\
\hline
\end{tabular}
\label{tab:strategies}
\end{table}
%
%
%
%
\section{split's assumption}
\subsection{Behavior on a Single Layer}\label{sec:appendix_multi_dupli}
The main assumption made in this article is hardware-based. 
In order for split to be efficient we need the following condition to be satisfied : "memory access and allocation should run faster that mathematical operations".
According to our research the validity of this condition is very dependent on the material and operation implementations.
The pruning performed corresponds to block sparsity as it removes blocks of computations but not according to regular structures used in modern implementations.
The current paradigm consists in optimizing matrix multiplications \cite{chetlur2014cudnn}.
And this affects the research on pruning which tend to focus on methods that leverage current hardware properties.
In this search we deviate from this both virtuous and vicious circle.
Nonetheless we empirically motivate \ours perspective by comparing it to regular convolutions with equivalent implementations in order to display the inference acceleration.
To do so we run a $3\times 3$ convolution with $5$ output channels on a random noise of shape $244 \times 224 \times 3$.
This convolution is pruned structurally for reference and pruned using split for comparison.
This test is highly influenced by the batch-size as well as the image size. 
We considered batches of size $10$.
The resulting plot are displayed in Fig \ref{fig:assumption}. 
We observe that split achieves lower yet decent acceleration with a similar implementation which should motivate the investigation on hardware and inference engines for such methods.
\begin{figure}[ht]
    \centering
    \includegraphics[width = 0.85\linewidth]{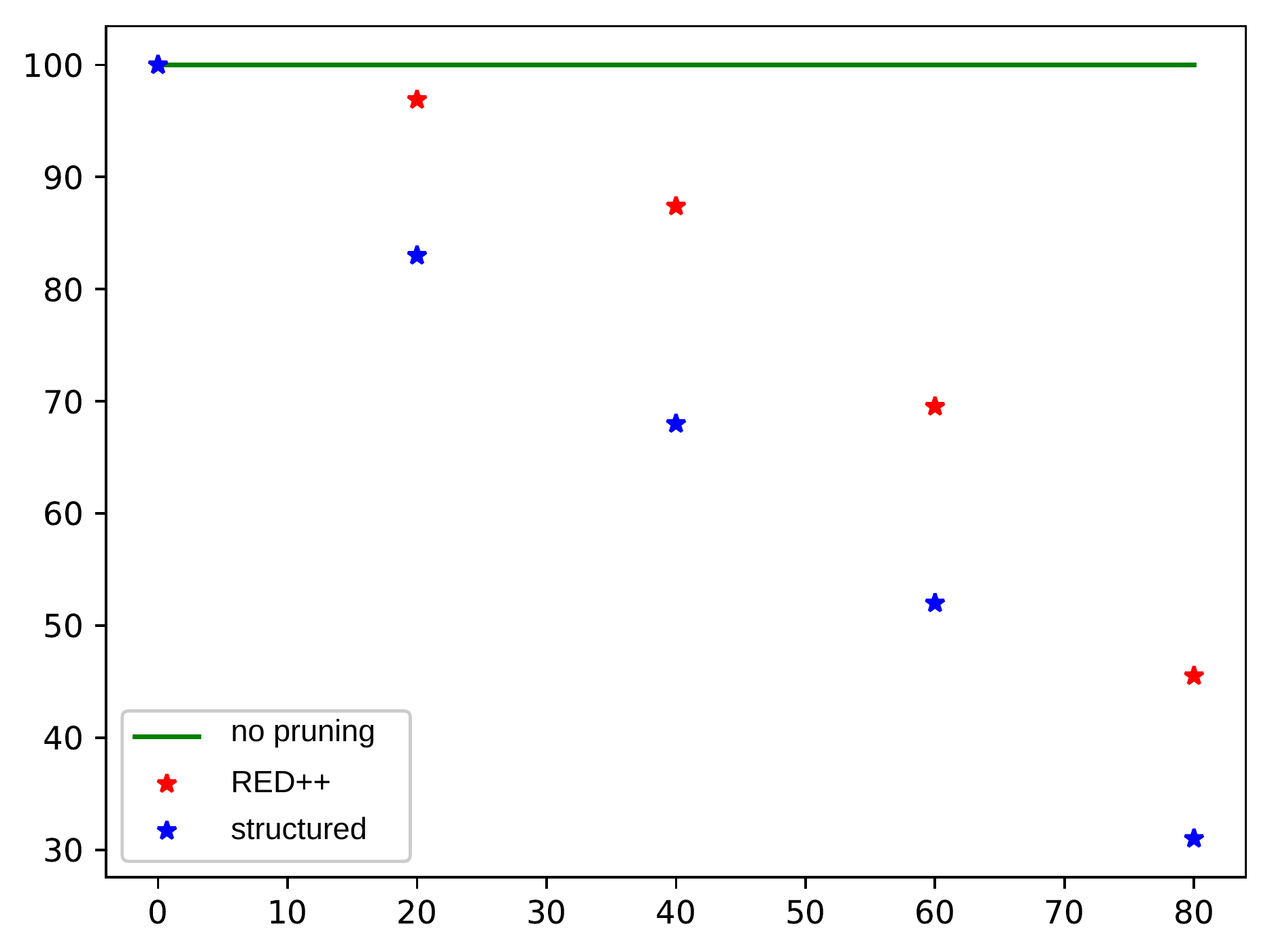}
    \caption{Plot of the percentage from the original runtime of \ours vs structured pruning on a CPU and batch of $10$ images of shape $224\times 224\times 3$.}
    \label{fig:assumption}
\end{figure}
%
%
%
%
\subsection{Behavior on the Full Model - FLOPs and Runtime}\label{sec:appendix_flops}
\begin{table}[th]
\renewcommand{\arraystretch}{1.15}
\caption{Comparison between the two most frequently used metrics in DNN pruning, the percentage of removed parameters and the percentage of removed FLOPs. The number of parameters removed are the same as the ones listed in Table \ref{tab:ablation}.}
\label{tab:FLOPs}
\centering
\begin{tabular}{|c|c|c|c|}
\hline
Dataset & Model & Params & FLOPs \\
\hline
\hline
\multirow{19}{*}{Cifar10} & 
 ResNet 20            & 67.48 & 65.57 \\
& ResNet 56           & 88.81 & 88.01 \\
& ResNet 110          & 91.82 & 91.14 \\
& ResNet 164          & 94.49 & 94.21 \\ 
\cline{2-4}
& Wide ResNet 28-10   & 80.13 & 79.24 \\
& Wide ResNet 40-4    & 68.59 & 68.72 \\
\cline{2-4}
& MobileNet V2 (0.35) & 55.48 & 54.94 \\ 
& MobileNet V2 (0.5)  & 59.70 & 60.15 \\ 
& MobileNet V2 (0.75) & 65.65 & 65.11 \\ 
& MobileNet V2 (1)    & 71.52 & 71.81 \\ 
& MobileNet V2 (1.4)  & 80.57 & 79.64 \\ 
\cline{2-4}
& EfficientNetB0      & 64.62 & 63.98 \\
& EfficientNetB1      & 68.75 & 67.51 \\
& EfficientNetB2      & 71.26 & 70.07 \\
& EfficientNetB3      & 72.67 & 71.75 \\
& EfficientNetB4      & 76.34 & 76.22 \\
& EfficientNetB5      & 80.30 & 80.29 \\
& EfficientNetB6      & 81.97 & 81.43 \\
& EfficientNetB7      & 83.42 & 83.61\\
\hline
\hline
\multirow{27}{*}{ImageNet} &
  ResNet 50           & 44.25 & 43.70 \\
& ResNet 101          & 44.51 & 43.85 \\
& ResNet 152          & 43.68 & 43.46 \\
\cline{2-4}
& MobileNet V2 (0.35) & 14.88 & 15.06 \\
& MobileNet V2 (0.5)  & 22.94 & 22.99 \\
& MobileNet V2 (0.75) & 35.75 & 34.56 \\
& MobileNet V2 (1)    & 46.97 & 46.00 \\
& MobileNet V2 (1.4)  & 85.42 & 85.34 \\
\cline{2-4}
& EfficientNetB0      & 54.43 & 54.32 \\
& EfficientNetB1      & 63.17 & 63.12 \\
& EfficientNetB2      & 68.35 & 68.17 \\
& EfficientNetB3      & 71.41 & 70.66 \\
& EfficientNetB4      & 74.58 & 72.75 \\
& EfficientNetB5      & 81.41 & 79.84 \\
& EfficientNetB6      & 87.88 & 86.96 \\
& EfficientNetB7      & 89.21 & 87.54 \\
\cline{2-4}
& DeiT T              & 70.32 & 68.97 \\
& DeiT S              & 84.93 & 82.84 \\
& DeiT                & 93.20 & 91.82 \\
\cline{2-4}
& CaiT XS24           & 81.85 & 80.65 \\
& CaiT S24            & 86.27 & 84.34 \\
& CaiT M36            & 91.80 & 89.84 \\
\cline{2-4}
& LeViT 128S          & 76.91 & 75.01 \\
& LeViT 128           & 77.46 & 75.85 \\
& LeViT 192           & 80.49 & 78.99 \\
& LeViT 256           & 85.93 & 84.37 \\
& LeViT 384           & 90.55 & 89.26 \\
\hline
\end{tabular}
\end{table}
An important question in evaluation is the choice of metric. 
In the case of DNN pruning it appears to be a particularly discussed matter as suggested in \cite{blalock2020state}.
The percentage of removed parameters seems to be the most common metric and the one we used in the main paper.
\newline\textbf{FLOPS: }Nonetheless, the percentage of removed FLOPs is an other commonly used metric which corresponds to the number of floating point operations to run.
In TABLE \ref{tab:FLOPs} we report our results on the two metrics to show how closely related these two measurements are in the case of \oursvir.
This similarity in the two metrics is often observed in DNN pruning but is not systematic.
\newline\textbf{Inference Time: }Considering that the goal of pruning is to speed-up inference-time an intuitive evaluation protocol for such methods would be to measure the speed-up in practice.
However this is almost never reported as it depends too much on the hardware setup (e.g. which device, CPU/GPU, batch size,...). 
Papers that uses such evaluations usually propose novel DNN architectures \cite{graham2021levit} or inference engine such as TensorRT \cite{vanholder2016efficient} or OpenVINO \cite{gorbachev2019openvino}.
Nonetheless we still made a case study of split runtime performance in Appendix \ref{sec:appendix_multi_dupli}.

\section{Algorithms}\label{sec:appendix_algo}
For clarity we summarize \ours in Algorithm \ref{alg:ours}. The only commutative steps are the merging step and split. 
\begin{algorithm}[ht]
  \caption{\ours method}
  \label{alg:ours}
\begin{algorithmic}
  \STATE {\bfseries Input:} trained DNN $f$ with weights $(W^l)_{l \in \llbracket 1 ; L \rrbracket}$ and $\alpha$
  \STATE $\tilde f\leftarrow$ Hashing\_step ($f$)\COMMENT{\textcolor{OliveGreen}{Algorithm \ref{alg:hashing}}}
  \STATE $\bar f\leftarrow$ Merging\_step ($\tilde f$, $\alpha$)\COMMENT{\textcolor{OliveGreen}{Algorithm \ref{alg:merging}}}
  \STATE $\hat f\leftarrow$ Splitting\_step($\bar f$)\COMMENT{\textcolor{OliveGreen}{Algorithm \ref{alg:split}}}
  \STATE {\bfseries return} $\hat f$ 
\end{algorithmic}
\end{algorithm}\newline
We apply the first two steps of RED which we recall in Section \ref{sec:preliminaries} as well as in Algorithm \ref{alg:hashing} and \ref{alg:merging}. To further prune the DNN we apply split (Algorithm \ref{alg:split}) which is the data-free structured pruning mechanism we detail in Section \ref{sec:preliminaries_method}. It consists in splitting each layer independently, by input to replace redundant sub-neurons computations by memory duplications.
\begin{algorithm}[!ht]
  \caption{Hashing\_step}
  \label{alg:hashing}
\begin{algorithmic}
  \STATE {\bfseries Input:} trained DNN $f$ with weights $(W^l)_{l \in \llbracket 1 ; L \rrbracket}$, hyperparameters $(\tau^l)_{l \in \llbracket 1 ; L \rrbracket}$
  \STATE Initialize $\tilde f = f$ 
  \FOR{$l=1$ {\bfseries to} $L$}
  \STATE $d^l = \text{KDE}(W^l)$
    \STATE extract ${(m_k^l)}_{k \in M^-}$ and ${(M_k^l)}_{k \in M^+}$ from $d^l$
    \STATE ${(M_k^l)}_{k \in M^+}\leftarrow$ NMS $\left({(M_k^l)}_{k \in M^+}, \tau^l\right)$
    \FOR{$w \in W^l$}
    \STATE find $k$ such that $w \in [m_k^l; m_{k+1}^l[$
    \STATE $\tilde w \leftarrow M_k^l$
  \ENDFOR
  \ENDFOR
  \STATE return $\tilde f$ 
\end{algorithmic}
\end{algorithm}
\begin{algorithm}[ht]
  \caption{Merging\_step}
  \label{alg:merging}
\begin{algorithmic}
  \STATE {\bfseries Input:} hashed DNN $\tilde f$, hyperparameters $(\alpha^l)_{l \in \llbracket 1 ; L \rrbracket}$
  \STATE Initialize $\bar f = \tilde f$ with $(\bar W^l)_{l \in \llbracket 1 ; L \rrbracket} \leftarrow (\tilde W^l)_{l \in \llbracket 1 ; L \rrbracket}$
    \FOR{$l=1$ {\bfseries to} $L-1$}
    \STATE $D \leftarrow$ matrix of $l^2$ distances between all neurons
    \STATE $d \leftarrow \alpha^l$ percentile of $D$ \hspace*{\fill}\textcolor{OliveGreen}{$\blacktriangleright$ $d$ is the threshold distance}
    \STATE $D_{i,j} \leftarrow$ $1_{D_{i,j} \geq d \text{ or } i =j}$ \hspace*{\fill} \textcolor{OliveGreen}{$\blacktriangleright$ graph of similarities $D$}
    \STATE $M \leftarrow$ connected components from $D$
  \STATE $\bar W^{l}_{\text{new}} = []$
    \FOR{comp $\in M$}
    \STATE $\bar W^{l}_{\text{new}}\text{.append}\!\!\left(\!\frac{1}{\left|\text{comp}\right|}\!\! \sum_{j\in\text{comp}} \!\bar W^{l}_{[...,j]}\right)$ \hspace*{\fill} \textcolor{OliveGreen}{$\blacktriangleright$ barycenters} 
  \ENDFOR
  \STATE $\bar W^l \leftarrow \bar W^{l}_{\text{new}}$
  \STATE $\bar W^{l+1}_{\text{new}} = []$ \hspace*{\fill} \textcolor{OliveGreen}{$\blacktriangleright$ layer $l+1$ update}
  \FOR{comp $\in M$}
  \STATE $\bar W^{l+1}_{\text{new}}\text{.append}\left(\sum_{i \in \text{comp}} \bar W^{l+1}_{[i,...]}\right)$
  \ENDFOR
  \STATE $\bar W^{l+1} \leftarrow \bar W^{l+1}_{\text{new}}$
  \ENDFOR
  \STATE return $\bar f$
\end{algorithmic}
\end{algorithm}
\begin{algorithm}[ht]
  \caption{splitting\_step}
  \label{alg:split}
\begin{algorithmic}
  \STATE {\bfseries Input:} pre processed DNN $\bar f$ with weights $(\bar W^l)_{l \in \llbracket 1 ; L \rrbracket}$
  \STATE $\hat f \leftarrow \bar f$
  \FOR{$l \in \llbracket 1 ; L \rrbracket$}
  \STATE new\_kernel = [], duplications = {}
  \STATE key $\leftarrow$ weights ID generator
    \FOR{$i \in \llbracket 1; n^{l-1} \rrbracket$}
        \STATE $\text{kernel}_i$ = []
        \FOR{$j \in \llbracket 1; n^{l} \rrbracket$}
        \IF{$W^l_{[:,:,i,j]} \in \text{kernel}_i$}
            \STATE duplications[key($W^l_{[:,:,i,j]}$)].append(j)
        \ELSE
            \STATE $\text{kernel}_i$.append($W^l_{[:,:,i,j]}$)
            \STATE duplications[key($W^l_{[:,:,i,j]}$)] = [j]
        \ENDIF
        \ENDFOR
      new\_kernel.append($\text{kernel}_i$)
    \ENDFOR
    \STATE $\hat f^l \leftarrow$ split\_layer(new\_kernel, duplications)
    \ENDFOR
    \STATE {\bfseries return} $\hat f$
\end{algorithmic}
\end{algorithm}\newline
\section{Qualitative Analysis}\label{sec:appendix_qualitative_analysis}
\begin{figure}
    \centering
    \includegraphics[width = \linewidth]{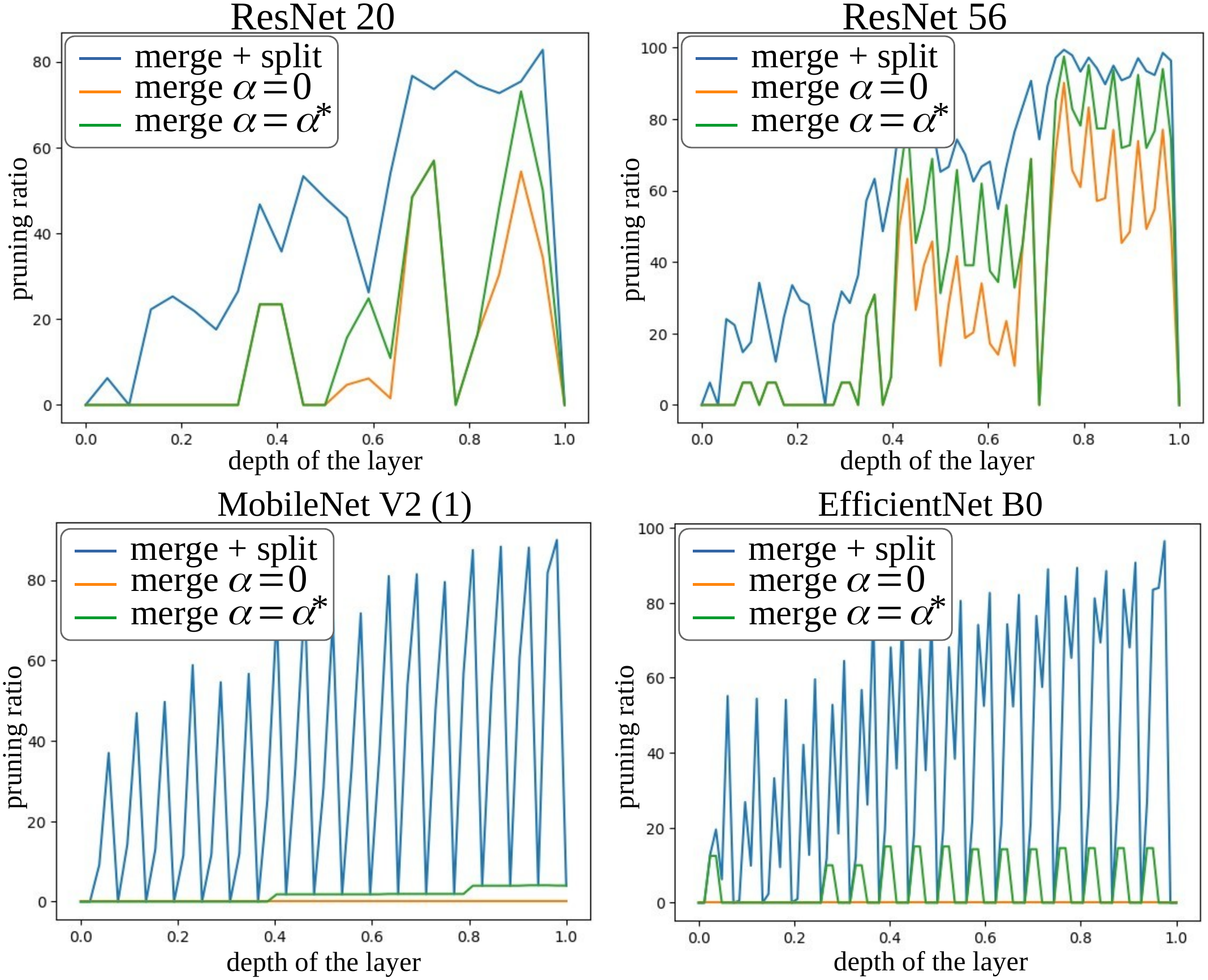}
    \caption{Layer-wise pruning ratios for ResNet 20 and 56 on Cifar10, MobileNet v2 with width multiplier $1$ and EfficientNet B0 on ImageNet. For each network, applying the merging and splitting steps systematically outperforms merge alone $\alpha = 0$ as well as $\alpha = \alpha^*$). Furthermore, it allows to prune certain layers such as the last ones in residual blocks.}
    \label{fig:qualitative_analysis}
\end{figure}
In Fig \ref{fig:qualitative_analysis}, we study the layer-wise pruning ratios for several networks trained on ImageNet. First, we observe that, the deeper the layer, the lower the compression. This is likely due to the fact that these layers usually have less weights, thus are likely to exhibit less redundancies than the shallower ones.
Second, with the merging step alone, we observe (Fig \ref{fig:qualitative_analysis} with the green and orange curves regularly going down to $0$) that some layers cannot be pruned, typically the last layer of a residual block in case of e.g. ResNet 20 and ResNet 56 (top row). This is a natural consequence of the definition of residuals blocks where the last layer's output is added to the input of the block, which thus has to be merged simultaneously. 
This adds a significant constraint on the condition of equation \ref{eq:merge} for neuron merging, hence very low pruning ratios.
However, this is not the case with splitting (blue curve) as its definition doesn't impact the consecutive layer of a specific pruned layer.
Third, depending on the architecture, the pruning from the merging step behaves differently, depending on the layer depth: linear for MobileNets and EfficientNets and ascending per block for ResNets. This observation allow us to design efficient guidelines for setting the layer-wise merging hyperparameters $\alpha^l$ (see Appendix \ref{sec:appendix_alpha} for more in-depth discussion on this subject): overall, the best performing strategy for setting the $(\alpha^l)$ is a block-wise strategy, for all architectures.

\end{appendices}


%

\



\ifCLASSOPTIONcaptionsoff
  \newpage
\fi

\end{document}